%% file: main.tex
\documentclass[accepted, breaklinks, colorlinks=True]{uai2026} % for initial submission
\usepackage{algorithm}
\usepackage{algpseudocode}
\input{math_commands.tex}

\usepackage[american]{babel}

\usepackage[utf8]{inputenc} % allow utf-8 input
\usepackage[T1]{fontenc}    % use 8-bit T1 fonts
\usepackage{hyperref}       % hyperlinks
\usepackage{url}            % simple URL typesetting
\usepackage{booktabs}       % professional-quality tables
\usepackage{amsfonts}       % blackboard math symbols
\usepackage{nicefrac}       % compact symbols for 1/2, etc.
\usepackage{microtype}      % microtypography
\usepackage{graphicx}       % figs
\usepackage{amsmath}
\usepackage{amssymb}
\usepackage{mathtools}
\usepackage{amsthm}
\usepackage{multicol}
\usepackage{multirow}
\usepackage[inline]{enumitem}

\newcommand{\blfootnote}[1]{%
  \begingroup
  \renewcommand\thefootnote{}\footnote{#1}%
  \addtocounter{footnote}{-1}%
  \endgroup
}

\theoremstyle{plain}
\newtheorem{theorem}{Theorem}[section]

\theoremstyle{definition}

\theoremstyle{remark}
\usepackage{natbib} % has a nice set of citation styles and commands
\usepackage{tikz} % nice language for creating drawings and diagrams

\title{Distilling Safe LLM Systems via Soft Prompts for On Device Settings}

\author[1]{Motasem Alfarra}
\author[1]{Cristina Pinneri}
\author[1]{Dana Kianfar}
\author[1]{Mohammed Almousa}
\author[1]{Christos Louizos}
\affil[1]{%
    Qualcomm AI Research
}
\begin{document}
\maketitle
\input{sections/0_abstract}
\input{sections/1_introduction}

\input{sections/2_related_work}
\input{sections/3_methodology}

\input{sections/4_experiments}

\input{sections/5_conclusions}
\bibliography{main}
% \bibliography{main}
% \bibliographystyle{iclr2026_conference}
\clearpage
\onecolumn
\appendix
% \title{Distilling Safe LLM Systems via Soft Prompts\\(Supplementary Material)}
% \maketitle
\input{sections/6_appendix}

\newpage

% \input{sections/0_abstract}
% This Supplementary Material should be submitted together with the main paper.

% \appendix
% \section{Additional simulation results}
% Table~\ref{tab:supp-data} lists additional simulation results; see also \citet{einstein} for a comparison. 

% \begin{table}[!h]
%     \centering
%     \caption{An Interesting Table.} \label{tab:supp-data}
%     \begin{tabular}{rl}
%         \toprule % from booktabs package
%         \bfseries Dataset & \bfseries Result\\
%         \midrule % from booktabs package
%         Data1 & 0.12345\\
%         Data2 & 0.67890\\
%         Data3 & 0.54321\\
%         Data4 & 0.09876\\
%         \bottomrule % from booktabs package
%     \end{tabular}
% \end{table}

% \section{Math font exposition}
% NOTE: necessary when ptmx or no mathfont class option is given
% \providecommand{\upGamma}{\Gamma}
% \providecommand{\uppi}{\pi}

\end{document}

%% file: math_commands.tex
\usepackage{amsmath,amsfonts,bm}

\def\eqref#1{equation~\ref{#1}}
\def\1{\bm{1}}

\def\eps{{\epsilon}}

\DeclareMathAlphabet{\mathsfit}{\encodingdefault}{\sfdefault}{m}{sl}
\SetMathAlphabet{\mathsfit}{bold}{\encodingdefault}{\sfdefault}{bx}{n}

%% file: sections/0_abstract.tex
\begin{abstract}
    Deploying safe large language models (LLMs) on resource-constrained edge devices presents a critical challenge: while dual-model systems combining LLMs with guard models provide effective safety guarantees, their substantial memory and computational demands make them prohibitively expensive for on-device deployment. This paper presents a comprehensive study of parameter-efficient safety alignment methods for resource-constrained settings. Through systematic evaluation across multiple LLM architectures, training objectives, and parameter-efficient fine-tuning approaches, we identify that \textbf{soft prompts combined with distillation-based training consistently outperform alternative methods}. We introduce distillation frameworks based on total variation and KL divergence that effectively transfer safety behaviors from guard models into learned soft prompts. Our evaluations on various benchmarks demonstrate that this combination achieves superior safety-usefulness trade-offs compared to LoRA adapters, steering vectors, and direct optimization methods, while requiring minimal additional memory and compute at inference time. These findings establish soft prompt distillation as the preferred approach for safety alignment in on-device LLM deployment.
\end{abstract}

%% file: sections/1_introduction.tex
\section{Introduction}

% Large Language Models (LLMs) have revolutionized the field of machine learning, enabling the resolution of complex real-world tasks. 
% For instance, LLMs can solve intricate mathematical problems that previously required human expertise. 
% They power chatbots that provide daily assistance, such as scheduling appointments or answering customer service queries. 
% Additionally, LLMs facilitate machine translation, allowing for seamless communication across different languages. 
% In the realm of computer vision, LLMs are embedded in applications like image captioning and object recognition. 
% These models have been integrated into various hardware chips, making them widely accessible in mobile phones and computers, thus enhancing the capabilities of modern technology.
\blfootnote{Corr. to: \texttt{ malfarra@qti.qualcomm.com}. Qualcomm AI Research is an initiative of Qualcomm Technologies, Inc}Despite their remarkable adoption across research and industry, large language models (LLMs) can generate unsafe and toxic content in response to certain prompts. 
For example, an LLM might produce harmful or offensive language if manipulated by a malicious user~\citep{xu2023llm, brundage2018malicious, liu2023autodan}.
% \dk{not the best example for motivation}. 

\begin{figure}[t]
    \centering
    \includegraphics[width=0.95\linewidth]{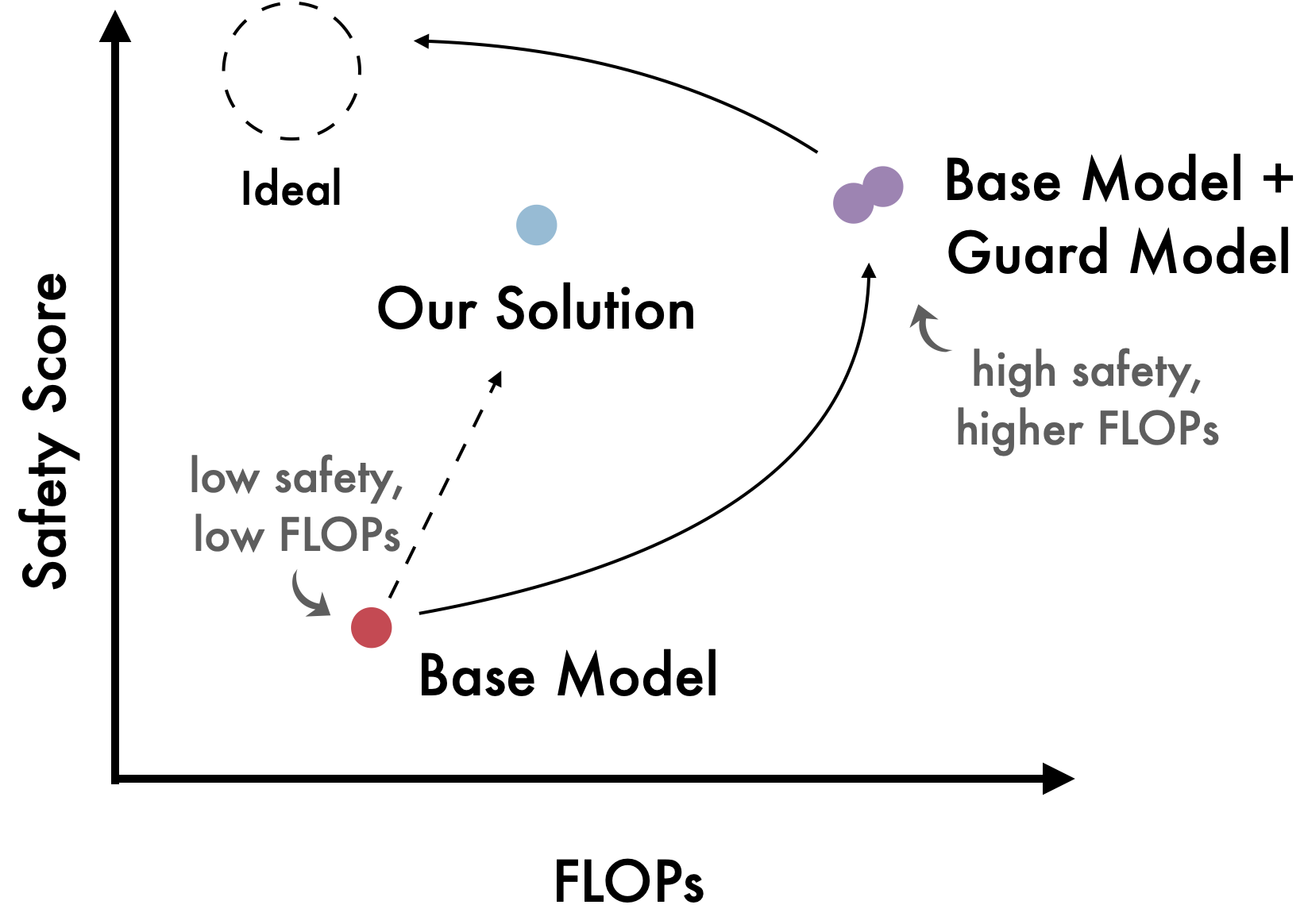}\hfill
    \caption{\textbf{Safety - Compute Trade-off.} LLMs (denoted as Base Model) can generate unsafe and toxic content. When paired with a Guard Model; altogether called a safe LLM system, their safety improves at the expense of a substantial compute and memory penalty which may hinder their usability. In this work, we systematically study safety alignment methods for on-device deployment and identify soft prompts with distillation as the optimal approach. }
    % \textbf{Right: Pipeline for our proposed TV-DiSP.} We distill a safe LLM system composed of a paired LLM and guard model into a set of learnable parameters equipped to the LLM.}
    \label{fig:pull}
\end{figure}

Safety fine-tuning methods such as reinforcement learning (RL), supervised fine-tuning, etc.~\citep{bai2022training}, can offer improvements in terms of safety alignment of the base LLM but system-level enhancements and layered defenses are necessary for minimizing risks~\citep{meta_responsible_ai}. %can be unreliable in production \textcolor{red}{cite? is it even true? This claim is a bit arbitrary, we probably need to remove it.}.
%This issue has highlighted a critical security problem, necessitating the development of LLM systems that are not only effective in their target domains but also safe for end users. 
To address this, guard models~\citep{llamaguard} have been introduced to evaluate and maintain the safety of LLM responses to user prompts. 
In this design the guard model assesses the safety of the response before exposing it to the user.
In a nutshell, a guard model is a separate LLM that classifies the input pair $(x, y)$ as safe or unsafe with $x$ being the user's prompt and $y$ being the response of the LLM.
When $(x,y)$ is deemed unsafe a pre-defined refusal answer, such as "Sorry, I cannot help with this matter.", overrides the initial response $y$.
This approach is the last line of defense against toxic and harmful responses, while preserving the capabilities of the LLM when its response is deemed safe. 
While recent studies~\citep{mangaokar2024prp} have shown that guard models are vulnerable to adversarial perturbations, they are used as a de-facto method for building safe LLM systems. 
% \dk{see comment}.

% As expected (and illustrated in Figure~\ref{fig:pull}) by combining an LLM with a guard model, 
However, the dual-model approach demands significant memory and computational resources, making it especially unsuitable for on-device deployment where memory and compute are severely limited~\citep{qin2024empirical}. 
This challenge is illustrated in Figure~\ref{fig:pull}.
% Running two large-scale models simultaneously is resource-intensive, making it impractical for mobile phones and other devices with limited computational power. 
In addition, the sequential nature of this approach (i.e. the guard waits for the full output of the LLM before classifying it) degrades important metrics such as time-to-first-token.
Various strategies have been proposed to address this issue, including quantizing the models to reduce memory consumption, distilling large LLMs into smaller models, and fine-tuning LLMs to mitigate toxic outputs~\citep{lin2024awq, fedorov2024llama}. 
While these methods improve memory efficiency and enhance safety, they often compromise the LLM's generalization capabilities~\citep{xu2024soft} and the effectiveness of different parameter-efficient fine-tuning (PEFT) methods and training objectives for safety alignment in resource-constrained settings remains unclear.

\begin{center} 
\textit{Which combination of adaptation method and training objective best balances safety, usefulness, and computational efficiency for on-device deployment?}
\end{center}
 
% This is because each approach typically focuses on either the LLM or the guard model in isolation, prioritizing either safety or efficiency.

In this work, we conduct a comprehensive study to answer this question.
We systematically compare different PEFT methods and training objectives to identify the most effective approach for on-device safety alignment.
Through extensive experiments across multiple LLM architectures and safety benchmarks, we find that soft prompts trained via distillation consistently achieve the best safety-usefulness trade-offs.
% \dk{terminology: safe llm system is too generic. Paired LLM-guard system?}.
% To enable this, we develop distillation frameworks based on total variation and KL divergence that effectively transfer safety behaviors from guard models into learned soft prompts, with theoretical guarantees on downstream task performance.
At test-time, the learned soft prompts are prepended to the user's prompt before feeding them to the LLM.
% We assess the efficacy of our approach by testing the safety and usefulness of the LLM responses on several benchmarks, architectures, and against other parameter efficient fine-tuning methods, providing consistently better results.
We validate our findings through on-device measurements on smart phones powered with Qualcomm Snapdragon hardware, demonstrating practical applicability for edge deployment.
Our contributions are thus three-fold:

\begin{enumerate}
    \item We demonstrate through systematic comparison that soft prompts trained via distillation consistently outperform LoRA adapters, steering vectors, and alternative training objectives (perplexity, policy gradients) for on-device safety alignment, achieving superior safety-usefulness trade-offs.
    \item We develop distillation frameworks based on total variation and KL divergence that effectively transfer safety behaviors from guard models to compact soft prompts, with guarantees on downstream  performance.
    \item We establish that soft prompt distillation provides practical on-device safety with less than 1\% memory overhead and less than 10\% compute overhead, dramatically outperforming the 2$\times$ cost of dual-model systems, validated across four LLM architectures and multiple safety benchmarks including on-device hardware measurements.

\end{enumerate}

%% file: sections/2_related_work.tex
\section{Related Work}
% Placeholder
\paragraph{LLM Safety.}
Recent studies have highlighted the susceptibility of large language models (LLM) to generating toxic or unsafe content with carefully designed prompts~\citep{mazeika2024harmbench, chao2024jailbreakbench, hartvigsen2022toxigen}, or when exposed to adversarial attacks~\citep{liu2023autodan, gong2025figstep, zou2023universal}. This has motivated researchers to explore various strategies to improve the safety alignment of LLMs. Among the most prominent approaches are Reinforcement Learning with Human Feedback (RLHF)~\citep{dong2024rlhf} and the use of auxiliary guard models~\citep{llamaguard, padhi2024granite}. Although these methods have shown promise in enhancing the safety of LLM outputs, they often come with significant drawbacks: RLHF requires costly training pipelines, and guard models can introduce substantial computational overhead during inference. In this work, we propose a parameter-efficient fine-tuning approach that distills the safety benefits of guard models into the base LLM, aiming to retain safety improvements while reducing inference costs.

\paragraph{Adapting LLMs}
Despite the impressive capabilities of recent large language models (LLMs) across a wide range of tasks, they often underperform when dealing with domain-specific knowledge or when their weights are quantized for on-device deployment. To address this performance gap, several parameter-efficient adaptation techniques have been proposed in the literature, including the widely adopted Low-Rank Adapters (LoRA)~\citep{hu2022lora, dettmers2023qlora}, steering vectors~\citep{turner2023activation, panickssery2312steering, wang2023trojan}, circuit breakers~\citep{zou2406improving}, and the more recent soft prompt tuning approach~\citep{xu2024soft, zheng2024prompt}. Among these, soft prompt tuning has shown significant promise in preserving model performance both before and after quantization. In this work, we investigate parameter-efficient fine-tuning methods—focusing particularly on soft prompt tuning—as a means to distill the safety capabilities of an LLM system equipped with a guard model back into the base LLM. This enables a more effective and computationally efficient alternative to deploying guard models at inference time.

%% file: sections/3_methodology.tex
\section{Methodology}
\paragraph{Preliminaries.}
\label{sec:preliminaries}
Let $p(y|x)$ represent an LLM that generates $y$ in response to a prompt $x$.
Further, let $p(s|x, y)$ represent a guard model that generates a safety label $s\in \{0, 1\}$ given the prompt-response pair $(x, y)$ where $s = 1$ represents the label ``safe" for the LLM's generation $y$.
A safe LLM system consists of both the LLM and guard model $p(y, s |x) = p(y|x) p(s|x, y)$.
This system returns to the user a response $r$ whose contents depend on the safety score of the pair $p(s|x, y)$.
We formulate the responses from the safe LLM system $p(r|x,y)$ as
\begin{align}
    \label{eq:safe_generation}
    % p(r|x,y) = 
    p(s=1 | x, y) \mathbb{I}(r = y) + p(s = 0| x, y) \mathbb{I}(r = y_r)
\end{align}
where the $y_r$ is a pre-defined refusal response such as ``Sorry, I cannot help with this matter." and $\mathbb{I}(.)$ is the indicator function. The safe LLM system output distribution can thus be formalized as 
\begin{equation}
 p(r|x) = \sum_y p(y|x)p(r|x, y).
\end{equation}
One major downside in deploying such a system is that it requires two full forward-passes through the LLMs (\emph{i.e.} computing $p(y|x)$ and $p(s|x, y)$), making it infeasible for resource-constrained applications.
%In this work, we make no assumptions on the LLM architecture, guard model, or training procedures previously undertaken.

\subsection{Distillation via Soft Prompts}\label{sec:tv_formulation}
In this section, we propose our novel adaptation strategy to distill the safe LLM system (described in Sec. {\ref{sec:preliminaries}}) to an instance of the LLM equipped with extra learnable parameters.
Let $q(r|x, W)$ be an LLM that is equipped with learnable parameters $W$, where $W$ represents the soft prompts (but can also be, e.g., LoRA parameters as in Sec.~\ref{sec:other_peft}). 
$W$ denotes a learned sequence of continuous prompt embeddings prepended to the input embeddings.
% while the base LLM remains frozen.
%Extending this notation, $q(r| x, W_{s=1})$ represents safe generations while $q(r| x, W_{s=0})$ represents unsafe generations.
%In this work we only focus on safe generations and therefore will simply denote $W_{s=1}$ as $W$.
%Note that by setting $W_s$ to be the guard model one can recover the original setup described in Equation~\ref{eq:safe_generation}.

% \paragraph{Why Total Variation}
% \begin{enumerate}
%     \item TV gives guarantess on downstream tasks
%     \item gor safety applications, guarantees are important.
%     \item TV is a powerful distillation scheme.
% \end{enumerate}

\paragraph{Total Variation Distillation}
% One of the primary contributions of this work is how the distillation framework described in \ref{sec:tv_formulation} is centered on the total variance distance.
The total variation distance is a suitable choice as the primary objective for our distillation because it provides probabilistic guarantees on how far the distilled model can deviate from the distillation target in terms of downstream task performance. 
% More specifically, 
We present the following theorem where the proof is left for the appendix.
\begin{theorem}
Let $p(r|x)$ be the safe system and $q(r|x, W)$ be the LLM equipped with soft prompts. We have that the performance gap between them on any test function $\phi(\cdot)$ with $|\phi(\cdot)|_\infty\leq 1$ is
\begin{align*}
     \left|\mathbb{E}_{p(r|x)}[\phi(r)] - \mathbb{E}_{q(r|x, W)}[\phi(r)] \right| \leq \\ 2 D_{TV}\left(p(r|x), q(r|x,W)\right),
\end{align*}
where $D_{TV}(\cdot, \cdot)$ is the total variation distance.
\label{theo:tv_guarantee}
\end{theorem}

% through its upper bound on the generalization error that is incurred due to distillation.
Having guarantees is especially desirable for safety-sensitive applications. Theorem~\ref{theo:tv_guarantee} can apply by considering $\phi(\cdot)$ as the safety probability / binary decision given by a model and/or human. 

As previously mentioned, we focus on the case where the learnable parameter $W$, i.e. the outcome of the distillation process, represent soft prompts. 
Once we have distilled the safe LLM system into these soft prompts $W$, they are prepended to the sequence of token embeddings of the user prompt and are fed into subsequent layers.
% Nevertheless our formulation is generic for any $W$.
% Here, we focus on the case where $W$ represents soft prompts prepended to the user's prompt in the embedding space before feeding them to the LLM layers.
% We begin by asking the following question: how should $q(r|x, W_s)$ behave?
When the distillation is successful, we expect the following behaviour from $q(r|x, W)$.
%one expects $q$ to approximate the entire safe LLM system described in Equation~\ref{eq:safe_generation}.
For safe responses, (\emph{i.e.} $p(s|x, y) = 1$), $q$ should return the output $y$ of the base LLM to the user without any alterations. 
This helps to preserve the utility of the underlying LLM.
% Note that the first case ensures the fluency of the equipped model, while the second case ensures the safety.
Otherwise for unsafe responses, (\emph{i.e.} $p(s|x, y) = 0$), $q$ should return the pre-defined refusal message (refer to Figure~\ref{fig:pipeline}). 
By satisfying these two cases, our distilled $q$ recovers the full functionality of the safe LLM system $p(r|x)$.
We optimize the learnable parameters $W$ to minimize the total variation distance between the two distributions $p(r|x)$ and $q(r|x, W)$ as follows:
\begin{align*}
    W^* &= \text{arg}\min_{W}\mathbb{E}_{x}\left[D_{TV}\left( p\left(r|x\right), 
    q\left(r|x, W\right)\right)\right],
\end{align*}
where the TV distance can be upper bounded as:

\begin{equation}
    D_{TV}(q, p) \leq 1 - \mathbb{E}_{p(y|x)p(r|x, y)}\left[\min\left(\frac{q(r|x, W)}{p(r|y, x)}, 1\right)\right]\label{eq:tv_refusal}
\end{equation}

\begin{figure}
    \centering
    \includegraphics[width=0.95\linewidth]{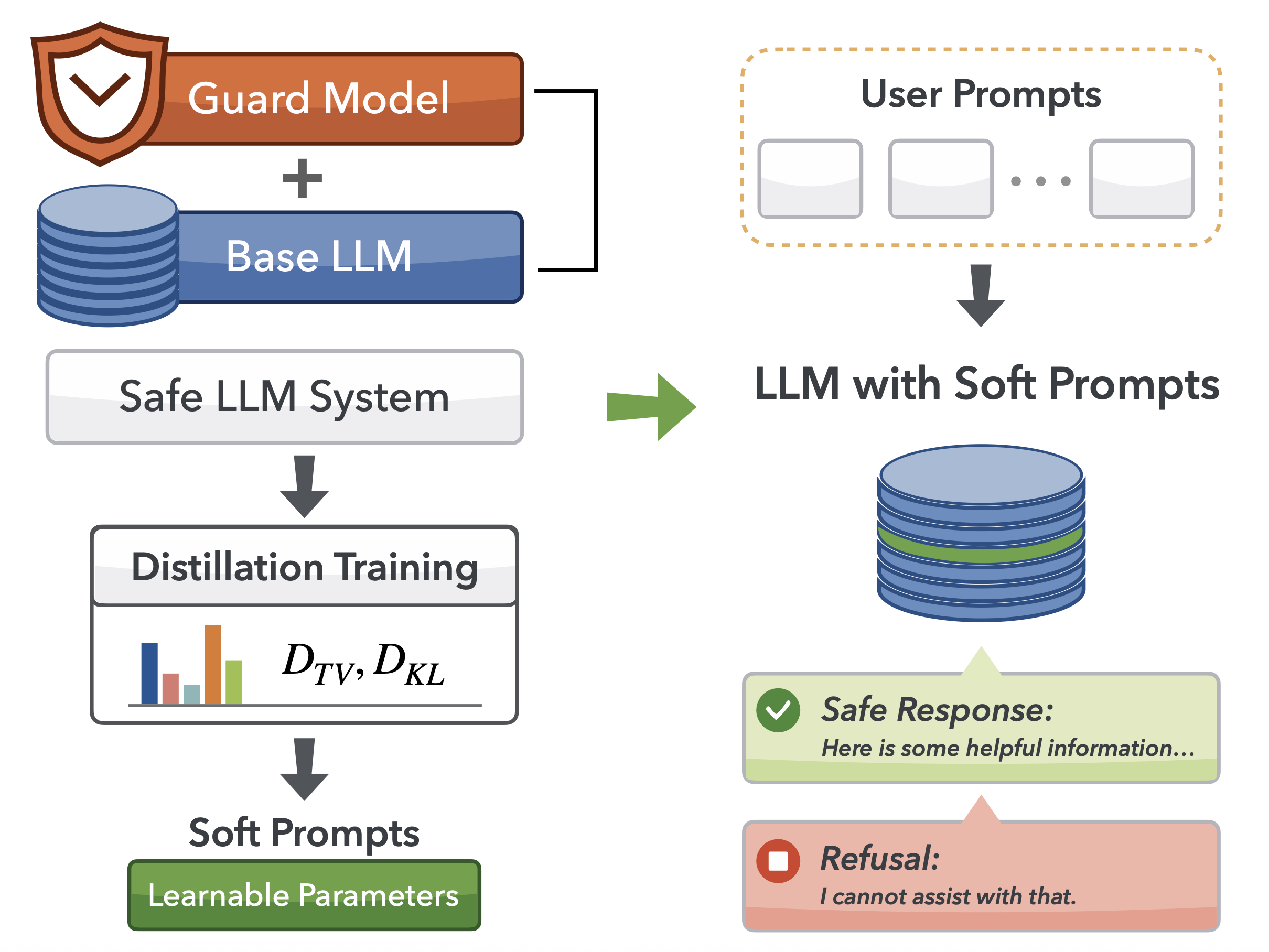}
    \caption{
    % \textbf{Left: Safety - Compute Trade-off.} LLMs (denoted as Base Model) can generate unsafe and toxic content. When paired with a Guard Model; altogether called a safe LLM system, their safety improves at the expense of a substantial compute and memory penalty which may hinder their usability. In this work, we distill safe LLM systems into a single model and shrink this penalty. 
    \textbf{Pipeline for our proposed TV-DiSP.} We distill a safe LLM system composed of a paired LLM and guard model into a set of learnable parameters (soft prompts) equipped to the LLM.}
    \label{fig:pipeline}
\end{figure}

While optimizing $D_{TV}$ is aligned with our objectives, the loss defined in Equation~\ref{eq:tv_refusal} can be hard to optimize due to operating on  probabilities directly.
It is thus easier to optimize the following objective function which relies on log-probabilities instead %(see Appendix \textcolor{red}{ref} for details).
% \begin{align}
% \begin{aligned}
% \label{eq:tv_refusal_upper_bound}
%     \max_{W}\,\,\, &\mathbb{E}_{p(y|x)}\Bigg[p(s=1| x, y) \left[\log\frac{q(r=y | x, W)}{p(s=1|x, y)} 
%     % - \log\left(q(r=y | x, W)\right) 
%     \right]_- \,  \\ +\,\, &p(s=0|x, y) \left[\log\frac{q(r=y_r |x, W)}{p(s=0|x, y)}
%     % - \log\left(q(r=\text{``refusal message"} |x, W)\right) 
%     \right]_-
%     \Bigg],
%     \end{aligned}
% \end{align}
\begin{align}
\begin{aligned}
\label{eq:tv_refusal_upper_bound}
    \max_{W}\, \mathbb{E}_{p(y|x)}\Bigg[p(s=1| x, y) \left[\log\frac{q(r=y | x, W)}{p(s=1|x, y)} 
    % - \log\left(q(r=y | x, W)\right) 
    \right]_- + \\ \, p(s=0|x, y) \left[\log\frac{q(r=y_r |x, W)}{p(s=0|x, y)}
    % - \log\left(q(r=\text{``refusal message"} |x, W)\right) 
    \right]_-
    \Bigg],
    \end{aligned}
\end{align}
where $p(s=1|x, y)$ and $p(s=0|x, y) = 1 - p(s=1|x, y)$ are the probabilities that $(x,y)$ is safe and unsafe respectively, and $[z]_- = \min(z, 0)$. 
The first term in Equation \ref{eq:tv_refusal_upper_bound} preserves the LLM response when it's deemed safe by the guard model while the second term learns the refusal message for unsafe responses.
Training $W^*$ in this fashion only requires a dataset of prompts without labels as the guard model dictates whether each prompt is safe or unsafe.
We denote our method Total Variation-based Distillation via Soft Prompts as TV-DiSP.

\paragraph{KL-Distillation.}
% At last, we replace our proposed TV approach with
Along TV distillation, we also study the effect of
minimizing the Kullback-Leibler distance between the safe LLM system and the LLM equipped with $W$ through
\begin{equation*}
\min_{W}\,\,\,\mathbb{E}_{x}\left[D_{KL}\left(p\left(r|x\right) , q\left(r|x, W\right) \right)\right].
\end{equation*}
We can show that this specific loss also provides guarantees on the downstream behavior, albeit looser than the ones we obtain with the total variation loss. More specifically, through an application of Pinsker's inequality~\citep{csiszar2011information}, we have the following simple upper bound
\begin{align*}
     \left|\mathbb{E}_{q(r|x, W)}[\phi(r)] - \mathbb{E}_{p(r|x)}[\phi(r)] \right| \leq \\ 2 D_{TV}\left(p(r|x), q(r|x,W)\right)\leq \\ \sqrt{2 D_{KL}(p(r|x),q(r|x, W))}.
\end{align*}
Therefore, the total variation loss is better in capturing differences in downstream performance comp ared to the KL divergence.
Empirically, we found that both TV and KL distillation schemes are quite effective.
% 
% as observed experimentally .
% \begin{figure}
%     \centering
%     \includegraphics[width=0.4\linewidth]{figures/pipeline.png}
%     \caption{\textbf{Pipeline for our proposed TV-DiSP.} We distill safe LLM System composed of LLM and a guard model into a set of learnable parameters eqipped to the LLM.}
%     \label{fig:pipeline}
% \end{figure}

\paragraph{Inference.}
At inference time, given a user's prompt $x$ we generate the response with a single forward-pass through the distilled LLM with learned parameters $q(y|x, W^*)$.
Note that in this forward-pass, the added compute and memory requirements for a moderately-sized $W^*$, e.g. 100 soft prompt vectors, are substantially lower than to what is required for two forward-passes when computing $p(y|x)$ and $p(s|x, y)$ in Equation \ref{eq:safe_generation}.
Through our experiments we will show that even a small $W^*$, e.g. 100 soft prompts consisting of a few thousand parameters, is sufficient to reduce the total variation distance to a sufficiently small value maintaining the LLM's fluency and the safety provided by the guard model.

\subsection{Other Optimization Schemes for Safety Alignment}\label{sec:other_losses}
Besides our proposed TV-distillation scheme, we explore the efficacy of other loss functions for this purpose.
In particular, we explore two strong baselines as competitors:

\paragraph{Perplexity Optimization.} 
Recently, \cite{xu2024soft} demonstrated how soft prompts can be trained to alleviate quantization-induced performance degradation by directly optimizing perplexity. 
As a baseline we explore perplexity optimization as an alternative to the total variation distance.
This entails learning $W$ by optimizing the perplexity of $q$ on a given dataset.
Formally, and following our notation, the perplexity optimization solves the following optimization problem:
\begin{equation*}
    \min_{W} \,\,\, - \log q(r=x_{t+1} | x_{1:t}, W) .
\end{equation*}
Note that this baseline follows the next token prediction (\emph{i.e.} $x_{t+1}$) when observing the $t$ tokens from the sequence ($x_{1:t}$).

\paragraph{REINFORCE.}
Next, we analyze an alternative baseline that optimizes for the safety score directly.
We follow the standard practice in the reinforcement learning literature by applying the log trick (REINFORCE) to calculate a tractable gradient through the following formulation:
\begin{equation*}
    \max_{W}\,\,\, \mathbb{E}_{q(y|x, W)} \,\, p(s=1| x, y) 
    % \,\,\, \rightarrow \,\,\, \nabla_W =  - p(s=1| x, y) \log\left(q(r=y | x, W)\right).
\end{equation*}
This baseline directly optimizes for the safety score of the model, measured by the guard model.

% something which we will observe empirically in our experiments.
% section. 
% It is worth mentioning that we employ a similar upper bound analysis to the KL optimization, similar to the one we used for the TV optimization. We leave the details of this analysis to the appendix.

\subsection{Extension to other PEFT Approaches}\label{sec:other_peft}
In previous sections, we assumed that $W$ is a set of soft prompts prepended in the embedding space to the user's prompt. 
Nonetheless, our formulation is generic to be applied other Parameter Efficient Fine-Tuning~(PEFT) methods such as Low Rank Adaptors and Steering Vectors.

\paragraph{Low Rank Adaptors~(LoRA).}
LoRA~\citep{hu2022lora} introduces trainable low-rank matrices that are injected into the attention and/or feed-forward layers of the transformer architecture. In our framework, the optimization objective over 
$W$ can be reinterpreted as learning these low-rank adapters, where the safety-aligned behavior is induced by constraining the latent representations via our proposed regularization. This allows LoRA to inherit the safety properties of soft prompt tuning while maintaining its parameter efficiency.
To ensure a fair comparison with soft prompt tuning, we set the rank of the LoRA adapters such that the total number of learnable parameters matches that of the soft prompts.

\paragraph{Steering Vectors~(SV).}
Steering vectors~\citep{turner2023activation} operate by linearly modifying the hidden states of the model to induce specific behaviors. Our method can be adapted to learn such vectors by treating 
W as a directional offset in the embedding or hidden space. The safety alignment is achieved by optimizing 
W to steer the model's responses toward desired safety criteria, effectively embedding behavioral constraints directly into the latent dynamics.

%% file: sections/4_experiments.tex
\begin{figure*}[t]
    \centering
    \includegraphics[width=0.45\linewidth]{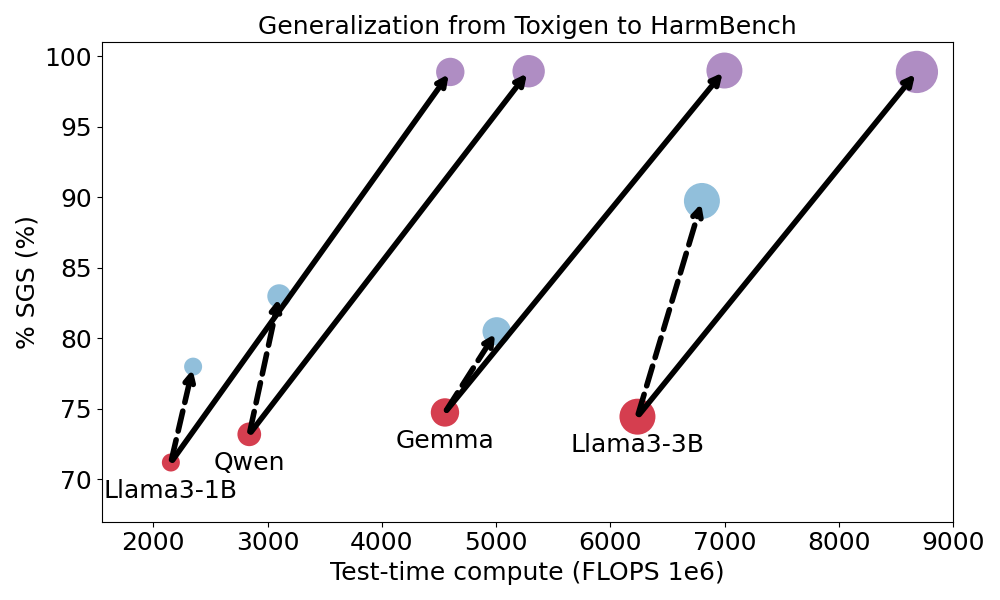}
    \hfill\includegraphics[width=0.45\linewidth]{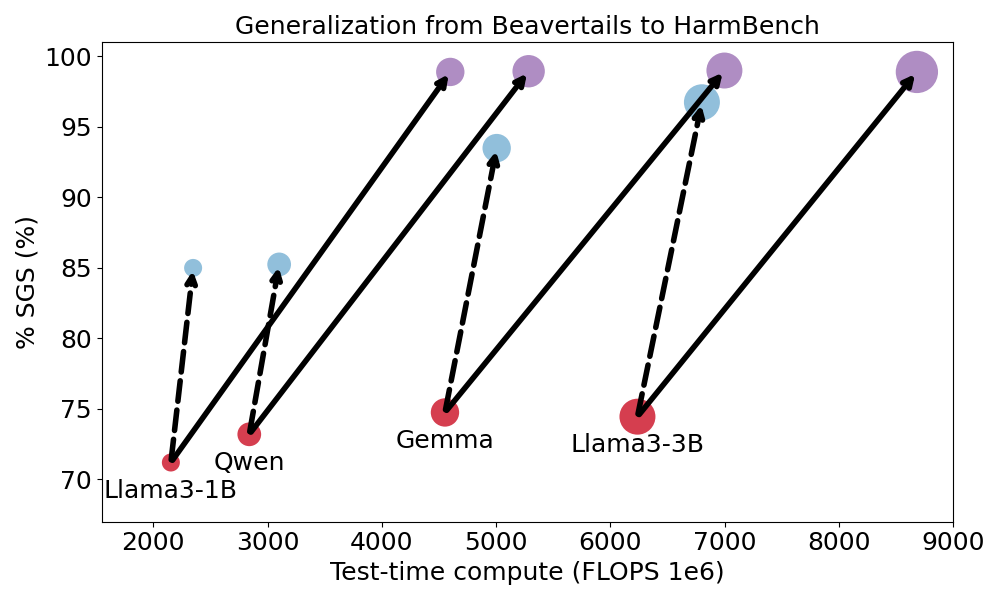} 
    % \vspace{-1cm}
    \includegraphics[width=0.5\linewidth, trim=0cm 2.4cm 0cm 2cm, clip]{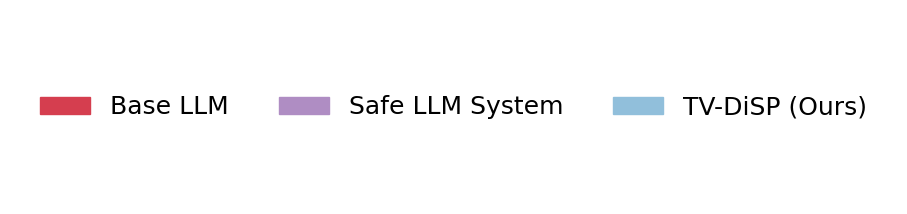}
    \caption{\textbf{Safety-Compute trade-offs when trained on Beavertails or Toxigen, and tested on HarmBench.} We report on the y-axis the Safety Guard Score~(SGS) according to LlamaGuard3-8B for three variations: the base LLM (red), the safe LLM system with LlamaGuard3-1B in-the-loop (purple), and our proposed distilled LLM with soft prompts (blue). The x-axis shows the test-time compute measured in the number of floating-point operations (FLOPs) to generate a single token for a context length of 512 on a fixed batch of data. The size of the circles represent the relative memory requirement for each variation. Our proposed TV-DiSP succeeds in distilling the safety of the safe LLM system with significant less memory and computation requirement.}
    \label{fig:distillation}\vspace{-0.2cm}
    % \vspace{-0.2cm}
\end{figure*}
\section{Experiments}
\subsection{Setup and Evaluation Protocol}\label{sec:exp_setup}
\paragraph{Models.} 
In our experimental setting, we focus on mimicking the on-device setting for when LLMs are deployed on edge devices.
To that regard, we run all our experiments by quantizing the weights of all models to 4-bits using the optimum-quanto library.
We experiment with four different models including Qwen2-1.5B~\citep{bai2023qwen}, Gemma2-2B~\citep{team2024gemma}, Llama3-instruct-1B, and Llama3-instruct-3B parameters~\citep{fedorov2024llama}.
Given the resource constraints typical of edge AI platforms, we selected smaller language models that are instruction-tuned and strike a good balance between performance and computational efficiency.
Furthermore, we use LlamaGuard3-1B as the guard model that provides the safety (\emph{i.e.} $p(s|x, y)$) score for distillation training and LlamaGuard3-8B to evaluate the distilled models \citep{llamaguard}.

\noindent\textbf{Evaluation Metrics.}
Since this work aims at studying safety-based LLM systems, we first assess the safety of the generation from the LLM before and after equipping it with the learnt $W$.
We leverage the state-of-the-art Llama3Guard-8B~\citep{llamaguard} parameter model to be the evaluator where we report the Safety Guard Score~(SGS) defined as:
\begin{equation}
    \label{eq:safety_score}
    SGS = \mathbb E_{x\sim \mathcal D}[ \mathbb{I}(p(s=1 | x, r) > 0.5) ],
\end{equation}
where $\mathcal D$ is the validation set of a given dataset, $x$ and $r$ are the prompt and its corresponding generation from LLM, respectively, and $\mathbb{I}$ is the indicator function.
Further, we compare the memory and computational needs to run different approaches such as the base LLM, the safe LLM system, and the LLM equipped with $W$.
In terms of computation, we report the FLOPs needed to generate a single token under a fixed context length of 512 (we leave to the appendix results under larger context length).
Further, we complement our evaluation paradigm to include measuring the usefulness of the LLM upon equipping it with the learned $W$.
To do so, we conduct the standard IFEval~\citep{zhou2023instructionfollowingevaluationlargelanguage} and GSM8K~\citep{cobbe2021gsm8k} datasets along with the standard 5-shot MMLU~\citep{hendrycks2020measuring} evaluation and report the accuracy of the model as a usefulness metric.
% It is worth mentioning that to verify the compatability of our learn

% Further, and to avoid falling into a reward hacking trap, we supplement the reported metrics with a perplexity measure to ensure that the equipped model is still fluent in responding to benign user's prompts.
% At last, and for our proposed TV-DiSP, we report the Total Variation Distance~(TVD) defined in Equation~\ref{eq:tv_refusal}.
% TVD represents an upperbound distance between the two probability distributions $p(r|x)$ and $q(r|x, W)$ where $p$ encapsulates both the LLM and the guard model, while $q$ represents the equipped LLM with the learnt parameters $W$
% Note that a TVD of zero means that both $p$ and $q$ have identical behavior.
% At last, report both the memory and compute reduction between using $p$ and $q$.

\begin{figure*}
    \centering
    \includegraphics[width=0.45\linewidth]{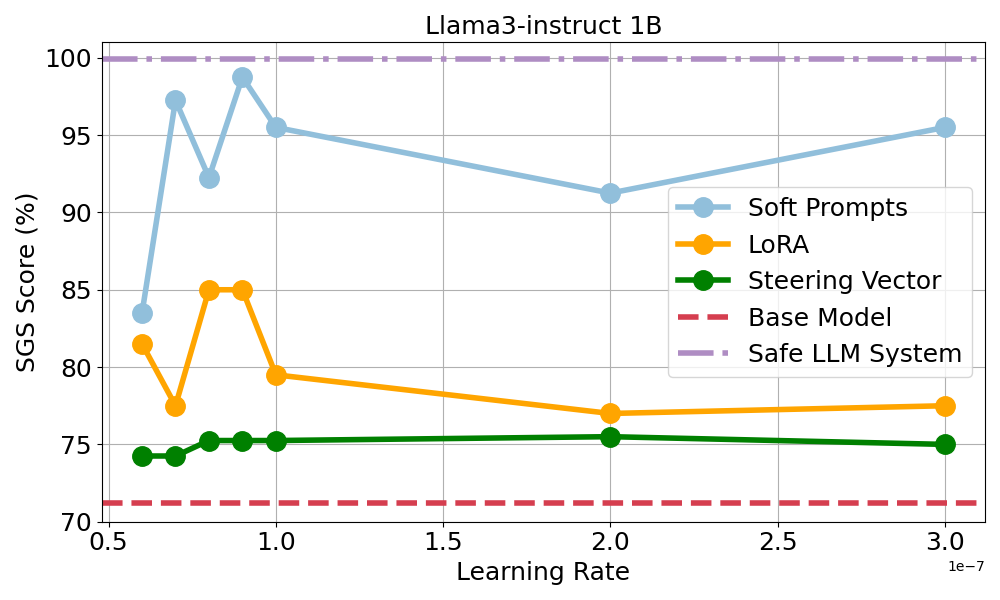}\hfill\includegraphics[width=0.45\linewidth]{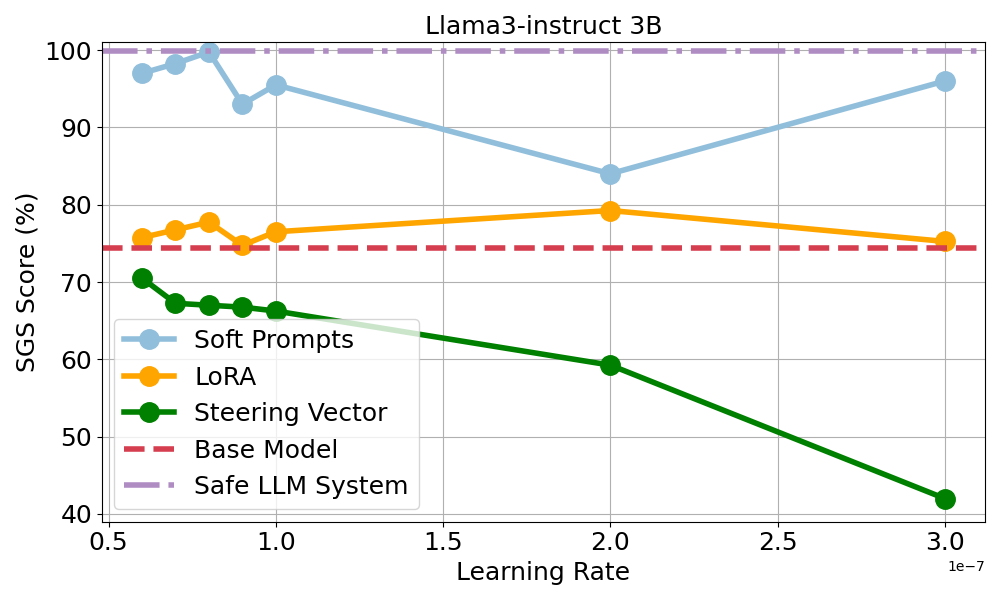}
    \caption{\textbf{Comparing TV-DiSP against TV-DiSV and TV-DiLoRA.} We employ our distillation scheme in Equation~\eqref{eq:tv_refusal_upper_bound} to distill the safe LLM system into a steering vector~(SV) or a low rank adaptor (LoRA). We conduct a single epoch training on Beavertails under different learning rates and report SGS on HarmBench. TV-DiSP consistently outperforms TV-DiSV and TV-DiLoRA.}
    % \vspace{-0.1cm}
    \label{fig:peft_comparisons}\vspace{-0.2cm}
\end{figure*}

\noindent\textbf{Datasets.}
Regarding the datasets, we experiment with training on the Beavertails~\citep{ji2023beavertails} dataset, where we subsample a fixed set of $10k$ prompts. 
Further, and to assess the generalizability of our approach, we also leverage the standard Toxigen~\citep{hartvigsen2022toxigen} dataset that includes both toxic and non-toxic prompts for training $W$.
In particular, we randomly subsample a fixed set of $5k$ prompts from the dataset and use them for the training experiments.
It is worth mentioning that in all our experiments, we conduct a single epoch of training (the model trains on each data point only once) for efficiency purposes.
To assess the reliability of the learned $W^*$, we conduct our safety evaluation on an out-of-distribution setting.
In particular, we experiment with the standard benchmark HarmBench~\citep{mazeika2024harmbench}, a collection of harmful adversarial prompts.
Moreover, we also include evaluations on Detect-JailBreak; a collection of three different datasets used for LLM safety evaluation~\citep{SCBSZ24, jailbreak, li2024wmdp, zou2023universal}.
At last, we leverage the test-set of Beavertails to include in-domain performance evaluation,
% By leveraging these datasets, we 
providing a comprehensive evaluation scheme.
Remaining of training details are in the appendix.
% We leave the remaining of training details to the appendix.
% All main findings are enumerated with roman numbers

% \paragraph{Training Details. 
% }
% In all our experiments, we conduct a single epoch of training (the model trains on each data point only once) for efficiency purposes.
% We set the learning rate to $1\times 10^{-7}$ and use the Adam optimizer.
% Unless stated otherwise, we assume that $W$ is a set of 100 soft prompts that are prepended to the user's prompt.

Unless stated otherwise, we refer to TV-DiSP as our method, set the architecture to Llama3.2-3B, and measure the safety with SGS on the HarmBench dataset.

\subsection{Recovering Safety with Distillation}\label{sec:dist_exps}

We first assess the efficacy of our proposed TV-DiSP in distilling the performance of a safe LLM system composed of the base LLM and the guard model.
Figure~\ref{fig:distillation} reports the results where the $x$-axis reports the computational requirements in FLOPs, the $y$-axis reports the safety guard score (SGS), and the diameter of each reported circle represents the relative memory requirement to store the deployed model on-device.
For this experiment, we analyzed 4 different LLMs namely; Llama3-1B, Qwen2-1.5B, Gemma2-2B, and Llama3-3B instruct tuned models.
We train the soft prompt on either Toxigen (left figure) or Beavertails datasets (right figure), where red, purple and blue circles represent the base LLM, the safe LLM systesm (LLM + Llama Guard 1B), and our proposed TV-DiSP.
In this experiment, we set $W$ to be a set of 100 soft prompts.

We observe \textbf{(i)} The safe LLM system can indeed identify unsafe generations by the LLM and correct them to a refusal response.
For example, the SGS of Llama3-insruct-1B model improves from 71\% to 99\%, measured by LlamaGuard-8B.
However, this safety gain comes at a big expense in both memory and computation.
For example, generating a single token from the base model requires $ 2.1\times10^{9}$ flops whereas the safe system requires $4.6\times10^{9}$ flops and doubles the memory requirements.
\textbf{(ii)} TV-DisP can successfully distill the safe LLM system into a single model equipeed with additional learned embeddings. 
%Our proposed TV-DiSP shines in providing a much better tradeoff.
For example, when $W^*$ is trained on Beavertails, TV-DiSP improves the safety of the base LLM by $20\%$ with less than 10\% additional computational cost, and less than $1\%$ additional memory consumption on both Gemma and Llama3-3B models.
It is noteworthy to mention that our experiments follow a challenging evaluation protocol by evaluating on out-of-distribution (mismatch between training and testing datasets).
That is, during the distillation phase, the model did not observe any adversarial prompts, similar to the ones in HarmBench.
This further strengthens the reliability and generalizability of the provided results.
\textbf{(iii)} Different training distribution can result in variation of the attained performance gain by TV-DiSP.
The is exemplified by changing the training distribution from Beavertails to Toxigen and conducting the same distillation scheme.
While TV-DiSP still provides consistent safety gains when compared to the base LLM, this performance improvement is enlarged with the better training distribution of Beavertails.
To that regard, in the rest of our experimentation in the paper, we conduct training with the stronger Beavertails dataset.
% TV-DiSP is consistent across all 4 base models and 2 different training distributions; it provides consistent safety gains with marginal additional computational overhead compared to employing the safe LLM system.
% It is noteworthy to mention that our experiments follow a challenging evaluation protocol by evaluating on out-of-distribution (mismatch between training and testing datasets), and basing the evaluation on a dataset containing adversarial prompts.
% This further strengthens the reliability and generalizability of the provided results.
\begin{figure*}[t]
    \centering
    \includegraphics[width=0.45\linewidth]{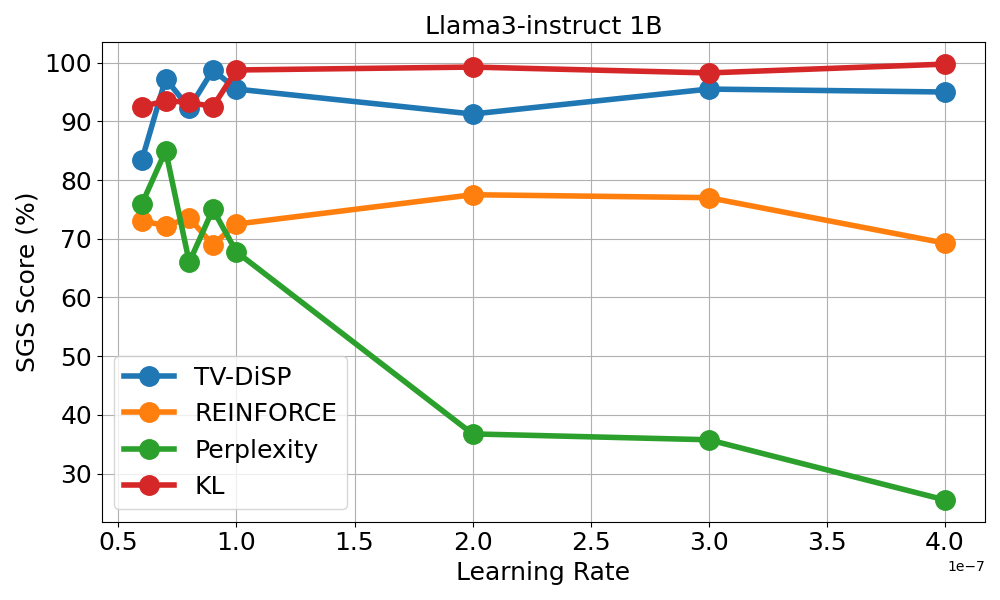}
    \hfill\includegraphics[width=0.45\linewidth]{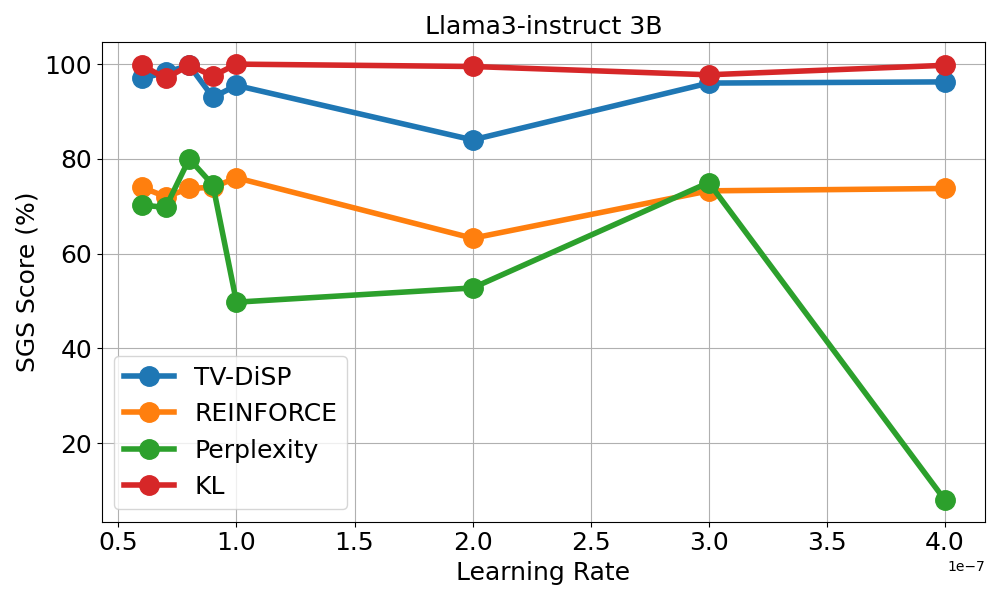}
    \caption{\textbf{Comparing TV-DiSP against other distillation schemes.} We compare our proposed total variation objective function to other loss functions in distilling the safe LLM system. We experiment with perplexity optimization, REINFORCE and KL divergence minimization. We report on the x-axis the learning rate used for training, the SGS on the y-axis on HarmBench. Left: Llama3-1B and  Right: Llama3-3B model is the base LLM.}
    \label{fig:loss_ablation}\vspace{-0.3cm}
\end{figure*}

\subsection{SP vs. LoRA and Steering Vectors}\label{sec:peft_exps}

Next, we set to study the efficacy of soft prompts as a parameter efficient fine-tuning method for distilling safe LLM system as compared to LoRA and Steering Vectors, dubbed as TV-DiLoRA and TV-DiSV, respectively. 
To do so, we employ our total variation distillation scheme described in Section~\ref{sec:tv_formulation}. For LoRA adapters, we set the rank to match the number of learnable parameters in the case of 100 soft prompts.
To alleviate the impact of training with sub-optimal learning rate for each method, we conduct the training on Beavertails with 7 different learning rates $[6, 7, 8, 9, 10, 20, 30] \times 10^{-4}$ and report the SGS on HarmBench in Figure~\ref{fig:peft_comparisons} for LLama3-1B and Llama3-3B models. 
The dashed lines represent the performance of the base model and the safe LLM system.
% as reference points.

We report \textbf{(iv)} Across al learning rates, soft prompts provide consistently the largest safety gains compared to LoRA adapters and Steering Vectors under both considered models.
In fact, the performance between TV-DiSP and TV-DiLoRA can grow larger than 20\%, as measured by the Llama-Guard-8B model.
\textbf{(v)} Our total variation distillation is a generally effective distillation scheme, and not applicable to just soft-prompt learning. This is demonstrated with the safety gains that TV-DiLoRA provides when compared to the base LLM.
\textbf{(vi)} Steering vectors do not have enough capacity to distill the guard model providing mixed performance; marginal improvement is observed in the Llama3-1B case, but performance deterioration is recorded in the Llama3-3B case.
We argue that soft prompts are preferable because they control behavior via input conditioning without altering the quantized backbone, while steering vectors lack capacity and LoRA is too intrusive for edge settings.

\subsection{Comparison Against Baselines}\label{sec:tradeoff}

% \begin{figure*}
%     \centering
%     \includegraphics[width=0.45\linewidth]{figures/iclr_figs/ablating_usefulness/llama3_instruct_1b_harmbench_soft_prompts.png}
%     \hfill\includegraphics[width=0.45\linewidth]{figures/iclr_figs/ablating_usefulness/llama3_instruct_3b_harmbench_soft_prompts.png}
%     \caption{\textbf{Comparing TV-DiSP against other distillation schemes in terms usefulness vs safety.} The x-axis reports the usefulness: 5-shot in context learning accuracy on MMLU benchmark, and the y-axis shows the SGS measured by Llama3 Guard - 8B for our proposed TV-DiSP, against REINFORCE and KL. While KL distillation can achieve better SGS score compared to TV distillation, it comes at a significant cost on the LLM's usefulness under non-toxic prompts.}
%     \label{fig:usefulness_vs_safety}
% \end{figure*}

Given the strong potential of distilling safe LLM systems into a few learnable embeddings, \emph{i.e.} soft prompts, we study the impact of different objective functions.
In particular, we explore learning $W$ with three other objective functions: perplexity optimization~(Perplexity), optimizing the safety score directly through policy gradient~(REINFORCE), and our KL and TV distillation schemes. Please refer to Section~\ref{sec:other_losses} for mathematical formulation details.
Similar to our setup in Section~\ref{sec:peft_exps}, we analyze two LLMs: Llama3-instruct 1B and 3B models, and train on Beavertails dataset.

Figure~\ref{fig:loss_ablation} shows the safety curves for each distillation method under different learning rates used in training, to alleviate suboptimal training hyperparameters.
We observe \textbf{(vii)} perplexity optimization provides small safety improvement under small learning rates. However, under relatively large learning rates, perplexity optimization degrades the SGS due to overfitting to the training distribution.
Similar to Perplexity, REINFORCE suffers from poor out of distribution generalization, providing marginal safety improvement on HarmBench.
\textbf{(viii)} KL and TV distillation succeed in distilling the safety behavior of the safe LLM system providing comparable SGS scores to each other.

\begin{figure*}
    \centering
    \includegraphics[width=0.45\linewidth]{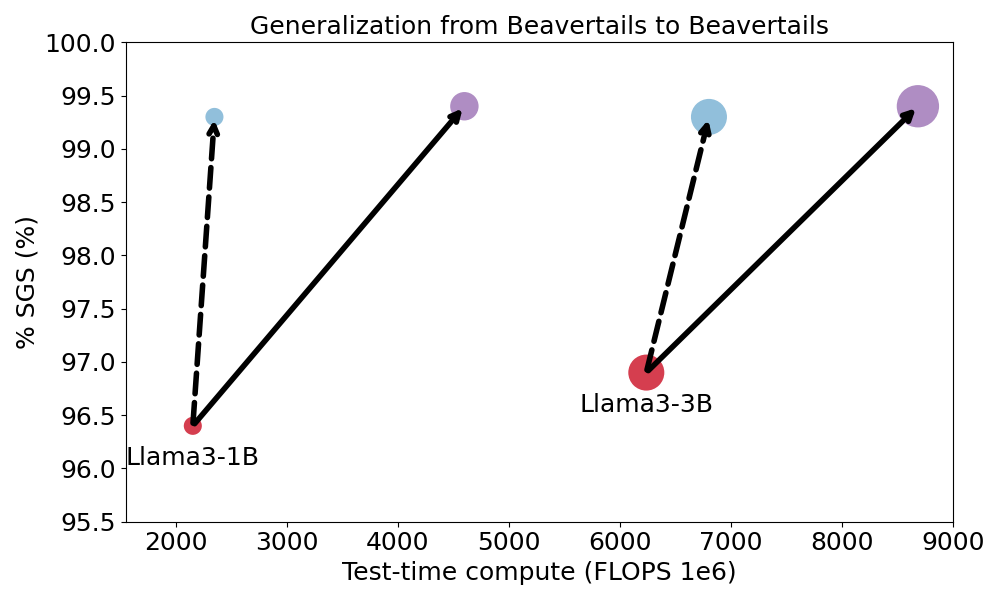}
    \hfill\includegraphics[width=0.45\linewidth]{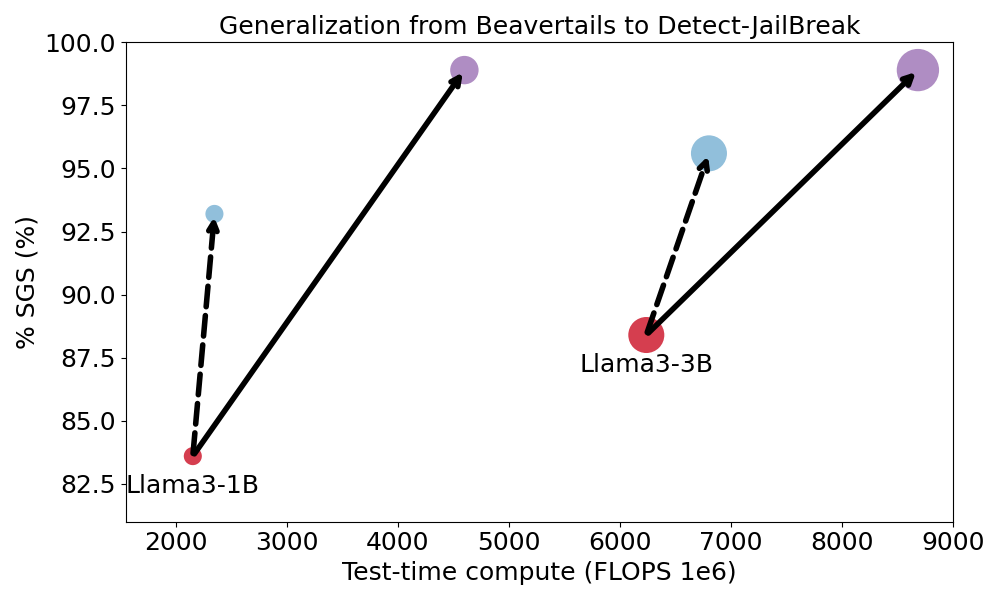}
    \includegraphics[width=0.5\linewidth, trim=0cm 2.4cm 0cm 2cm, clip]{figures/iclr_figs/distillation_results/legend.png}
    \caption{\textbf{Generalization to in-distribution and out-of-distribution.} Left: Results on Beavertails test-set (in distribution). Right:Results on Detect-Jailbreak dataset. Our proposed TV-DiSP provides consistent safety gains on both in- and out of- distribution settings on two different LLM architectures.} 
    \label{fig:other_datasets}
\end{figure*}

\subsection{Measuring Usefulness on Device}\label{sec:usefulness_measure}

To complement the safety results in the previous section, we assess whether introducing soft-prompt controls affects model usefulness under non-toxic inputs. While our earlier analysis showed that KL- and TV-based distillation strategies achieve comparable SGS, it is critical to quantify any utility degradation in realistic on-device settings.

\begin{table}[!b]
\centering
\caption{\textbf{Measuring usefulness on device on IFEval (0 shot) and GSM8k (5 shots) benchmarks}. Higher is better. TV and KL distillation provide the best tradeoff of safety gains with minimal usefulness drop. measurements are executed on device on a smartphone with a \textit{Qualcomm Snapdragon 8 Elite Gen 5} chipset. }
\label{tab:ifeval-gsm8k}
\begin{tabular}{lcccc}
\toprule
\multirow{2}{*}{\textbf{Method}} & \multicolumn{2}{c}{\textbf{IFEval}} & \multicolumn{2}{c}{\textbf{GSM8k}} \\
\cmidrule(lr){2-3}\cmidrule(lr){4-5}
 & Instance & Prompt & Flexible & Strict \\
\midrule
Base LLM            & 68.9 & 58.4 & 52.5 & 52.2 \\
Perplexity                  & 42.7 & 28.6 & 19.3 & 19.3 \\
Reinforce                  & 70.7 & 60.0 & 54.9 & 53.8 \\
KL-DiSP                  & 67.7 & 57.0 & 45.1 & 45.0 \\
TV-DiSP                  & 65.4 & 53.5 & 48.5 & 49.6 \\
\bottomrule
\end{tabular}
\end{table}

\textbf{Setup.} We evaluate the base LLM and its soft-prompt–equipped variants on {IFEval} (instruction adherence) and {GSM8k} (grade-school math reasoning). All measurements are executed on device on a smartphone with a \textit{Qualcomm Snapdragon 8 Elite Gen 5} chipset. Notably, the model and soft prompts are deployed with {4-bit quantization}, which makes integration seamless and preserves the effectiveness of the learned prompts without incurring additional latency or accuracy loss in practice.

\textbf{Results.} Table~\ref{tab:ifeval-gsm8k} summarizes the findings. The Perplexity-based method suffers substantial degradation across both benchmarks, indicating that it overfits to low-likelihood regions and consequently fails to preserve usefulness. By contrast, {Reinforce} maintains strong utility (close to the base model in both IFEval and GSM8k), but it does not improve SGS in our safety evaluations, limiting its effectiveness as a safety optimizer.
On the other hand, we observe that \textbf{(ix)} the two distillation variants -- KL-DiSP and TV-DiSP -- exhibit consistently favorable trade-offs: they deliver large safety gains (high SGS) while incurring only a minimal loss in usefulness relative to the base model. On IFEval, KL-DiSP remains slightly closer to the base model, whereas TV-DiSP is competitive on GSM8k, including under the `strict' evaluation. Overall, these results indicate that \textbf{TV-DiSP and KL-DiSP provide the best safety–usefulness balance} for on-device deployment.
Further experimental details along with results on MMLU are in the appendix.

% To address this question, we perform a 5-shot in-context learning evaluation on the standard MMLU~\citep{hendrycks2020measuring} benchmark and report accuracy as a proxy for usefulness.
% To answer this question, we conduct a 5-shot in context learning evaluation on the defacto MMLU~\citep{hendrycks2020measuring} benchmark and report the accuracy as a usefulness measure. \cp{see comment}
% For each optimization framework, we select the soft prompts that yield the best performance (based on the optimal learning rate) to ensure a fair comparison.
% Figure~\ref{fig:usefulness_vs_safety} illustrates the trade-off between usefulness and safety for the Base LLM, the Safe-LLM system, our proposed TV-DiSP method, and the REINFORCE and KL distillation schemes.
% We observe that \textbf{(ix)} our proposed TV-DiSP achieves the best balance between safety and usefulness, demonstrating a more effective distillation approach than KL.
% Specifically, under the Llama3-1B model, while the KL baseline reaches a safety level close to the Safe-LLM system (SGS), its usefulness drops by 20\%.
% Finally, these results highlight the need for even stronger distillation strategies that can further improve safety without compromising usefulness relative to the Base LLM.

% \subsection{Ablation Studies.}
\subsection{Other Safety Benchmarks}
In all our previous experimentation, we focused our evaluation on the standard HarmBench dataset.
In this section, we explore the efficacy of our proposed TV-DiSP under two different settings: the easy setting of evaluating in-distribution and the challenging jailbreak setting.
For the first setting, we conduct our evaluation on the test-set of Beavertails~\cite{ji2023beavertails} dataset.
For the second setting, we leverage a subset of the Detect-Jailbreak~\citep{SCBSZ24, jailbreak, li2024wmdp, zou2023universal} benchmark, which is a collection of three different datasets used to evaluate LLM safety.
In particular, we leverage a subset of Detect-JailBreak where all prompts are labeled as jailbreaks. 
We feed these prompts to Llama3 1B and 3B models and record the SGS measured with Llama3-Guard-8B model.
Figure~\ref{fig:other_datasets} summarizes the results.

We observe: \textbf{(x)} Our proposed TV-DiSP provides consistent performance improvement under both scenarios by successfully distilling the safe LLM system.
In particular, and under the challenging Detect-JailBreak benchmark, We improve the safety score SGS by more than 5\% under two different LLMs.
Furthermore, the safety improvement provided by TV-DiSP is also observed on the easier in-distribution setting with a consistent safety enhancement of more than 1\%.
These results complement our findings on the efficacy of our proposed method and further shows the generalization of our TV-DiSP under different testing settings.

\paragraph{Section Summary.}
In this section, we conducted a comprehensive experimental evaluation of our proposed TV and KL approaches in distilling safe LLM systems.
We showed the generalizability of our approach under different architectures and training distributions \textbf{(i-iii)}, its superiority when compared to other parameter efficient fine-tuning methods \textbf{(iv - vi)}, its advantages when compared to other safety optimization schemes \textbf{(vii - ix)}, and finally its consistency under different evaluation benchmarks \textbf{(x)}.
We leave to the appendix further experiments including ablating the impact of changing the number of learned soft prompts (i.e. the size of $W$) and optimizing $W$ with PPO~\citep{schulman2017ppo}.

%% file: sections/5_conclusions.tex
\section{Conclusions}

This paper addresses the challenge of deploying safe LLMs in resource-constrained, on-device settings. Through extensive empirical evaluation, we show that soft prompts trained via distillation consistently outperform LoRA adapters, steering vectors, and alternative objectives across safety and usefulness metrics. Our total variation– and KL-based distillation frameworks effectively transfer safety behaviors from guard models while preserving model utility. Experiments across multiple LLM architectures (Llama3, Qwen2, Gemma2), safety benchmarks (HarmBench, Beavertails, Detect-JailBreak), and usefulness datasets (IFEval, GSM8K, MMLU) demonstrate substantial safety gains with minimal overhead—under 1\% additional memory and less than 10\% additional compute—far outperforming dual-model approaches. On-device measurements further confirm the practicality of soft prompt distillation, establishing it as a strong solution for safety alignment in edge deployments.

%% file: sections/6_appendix.tex
\section{Methodology}
\subsection{Proof of Theorem~\ref{theo:tv_guarantee}}
In section~\ref{sec:tv_formulation}, we provided a theoretical statement on the probabilistic guarantees that the TV distillation approach provides.
In this section, we provide its proof.
\begin{theorem}[restatement]
Let $p(r|x)$ be the safe system and $q(r|x, W)$ be the LLM equipped with soft prompts. We have that the performance gap between them on any test function $\phi(\cdot)$ with $|\phi(\cdot)|_\infty\leq 1$ is
\begin{align*}
     \left|\mathbb{E}_{q(r|x, W)}[\phi(r)] - \mathbb{E}_{p(r|x)}[\phi(r)] \right| \leq  2 D_{TV}\left(q(r|x,W), p(r|x)\right),
\end{align*} 
where $D_{TV}(\cdot, \cdot)$ is the total variation distance.
\end{theorem}
\begin{proof}
The statement is a direct consequence of the sup representation of the total variation distance~\citep{polyanskiy2014lecture}
\begin{align*}
    D_{TV}(q, p) = \frac{1}{2}\sup_{\{\phi, |\phi|_\infty\leq 1\}}|\mathbb{E}_{q}[\phi] - \mathbb{E}_{p}[\phi]| 
     \geq \frac{1}{2}|\mathbb{E}_{q}[\phi] - \mathbb{E}_{p}[\phi]|,
\end{align*}
for $\{\phi, |\phi|_\infty\leq 1\}$.
\end{proof}

\subsection{Derivation of Equation~\eqref{eq:tv_refusal}}
Next, we derive the upper-bound of the total variation distance showed in Equation~\eqref{eq:tv_refusal}. This upper-bound is useful for facilitating the optimization of the total variation distance. 
\begin{align*}
D_{TV}(q, p) & = \frac{1}{2}\sum_r |q(r|x, W) - p(r|x)| \\
& = \frac{1}{2}\sum_r\left|\mathbb{E}_{p(y|x)}\left[q(r|x, W) - p(r|x,y)\right] \right|\\
& \leq \mathbb{E}_{p(y|x)}\left[D_{TV}\left(q(r|x, W), p(r|x, y)\right)\right]\\
& = 1 - \mathbb{E}_{p(y|x)}\left[\sum_r \min(q(r|x, W), p(r|x, y))\right] \\
& = 1 - \mathbb{E}_{p(y|x)p(r|x, y)}\left[\min\left(\frac{q(r|x, W)}{p(r|y, x)}, 1\right)\right].\label{eq:tv_refusal}
\end{align*}

\subsection{Extension Beyond Binary Safety Labels}
While our current work focuses on the practically important case where the safety variable $s$ is binary---i.e., safe versus unsafe---the formulation can be naturally extended to a more fine-grained setting in which $s$ is a categorical variable representing multiple safety outcomes, such as subtle toxicity, hallucination, privacy risk, or other unsafe behaviors.

In such a setting, the response of the safe system can be defined by assigning a category-specific action to each safety label. For example,
\begin{equation*}
p(r \mid x, y) = \sum_{c \in \mathcal{C}} p(s = c \mid x, y)\,\mathbb{I}\!\left[r = a_c(y)\right],
\end{equation*}
with
\begin{align*}
a_{\text{safe}}(y) &= y, \\
a_{\text{subtle toxicity}}(y) &= \text{detoxify}(y), \\
a_{\text{hallucination}}(y) &= \text{``I do not know the answer.''}, \\
a_{\text{other unsafe}}(y) &= \text{``I am sorry, I cannot help with that.''}.
\end{align*}

Here, $\text{detoxify}(\cdot)$ denotes a post-processing operation that rewrites the response to remove toxic content while preserving as much useful information as possible.

% \subsection{Algorithm Describing TV/KL-DiSP}

\vspace{-0.3cm}\section{Experiments}

\input{algorithms/algo}

\subsection{Setup and Evaluation Protocol - Extended}\vspace{-0.1cm}
In section~\ref{sec:exp_setup}, we provided the crucial details for our experimental setup.
Given the space limiataiton, and for transparency and full reproducibility, we provide the rest of the details for our setup and evaluation protocol in this section.

\textbf{Training Details.} In all our trainings, we fixed Adam~\citep{kingma2017adam} to be the optimizer in action with $\eps=10^{-7}$.
Regarding LoRA: We set the rank to 2 or 3 that matches the number of learnable parameters in the set of soft prompts.
Regarding SV: We apply the learnt steering vector to the output of layer 13, following the standard practices~\citep{panickssery2312steering}.
Algorithm~\ref{alg:tv-disp} summarizes the implementation of TV-DiSP. Note that for KL-DiSP, line 9 is replaced with the KL divergence loss outlined in Section~\ref{sec:tv_formulation}.

% At last, we multiply the learning rate by a factor of $0.1$ half way through the training.

\textbf{Evaluation Details.}
Regarding the evaluation dataset: for HarmBench, we leveraged all the 400 available prompts in the evaluation.
For Detect-Jailbreak, we leverage a subsample of 500 prompts that are both labeled as jailbreak prompts and regularly constructed (not adversarial prompt injection).
For the evaluation on Beavertails, we sub-sampled a fixed set of $1k$ prompts from the test set.
For the choice of learning rate in Section~\ref{sec:tradeoff}, we defined the optimal learning rate to be the one with the best sum of safety and utility (i.e. SGS+MMLU Accuracy).

\subsection{Ablating the Size of W}
Throughout our main experiments, we fixed the size of the soft prompt set $W$ to 100 vectors, which are prepended to the user’s prompt during inference. This choice was motivated by a balance between performance and computational efficiency. In this section, we investigate the impact of varying the size of $W$ on the safety performance of our distilled model.
To this end, we replicate the experimental setup from Section~\ref{sec:dist_exps} and train soft prompts of varying sizes: \{10, 50, 150, 200\}, using the Beavertails dataset. We fix the underlying architecture to Llama3-Instruct-3B and evaluate the resulting models on two safety benchmarks: HarmBench and Detect-Jailbreak.

Figure~\ref{fig:size_w} presents the results, where the $x$-axis denotes the number of learned soft prompts and the $y$-axis reports the Safety Guard Score (SGS) as measured by LlamaGuard-8B. As expected, increasing the number of soft prompts enhances the model’s capacity to approximate the behavior of the safe LLM system, leading to improved safety scores.
% across both benchmarks.
However, this improvement comes at a cost. Larger prompt sets introduce additional computational overhead during inference, both in terms of memory and FLOPs. Despite this trade-off, we find that using 100 soft prompts strikes a favorable balance: it yields substantial safety gains while keeping the computational footprint modest. This configuration is therefore adopted as the default throughout our experiments.

\begin{figure*}[t]
    \centering
    \includegraphics[width=0.455\linewidth]{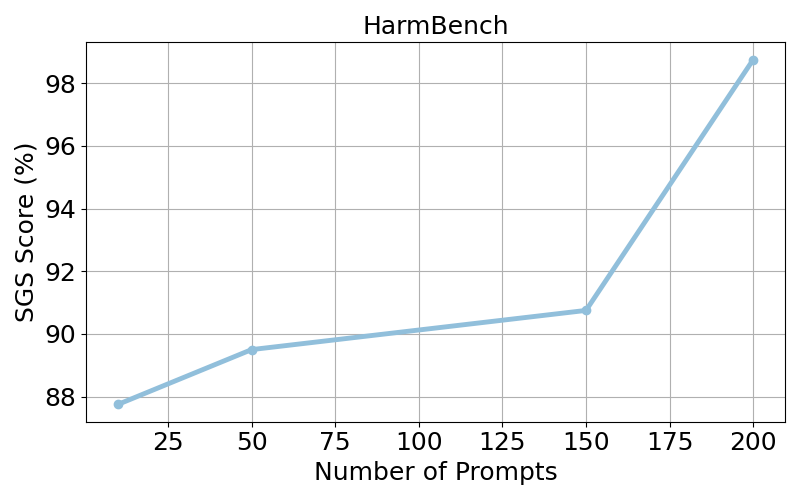}
    \hfill\includegraphics[width=0.455\linewidth]{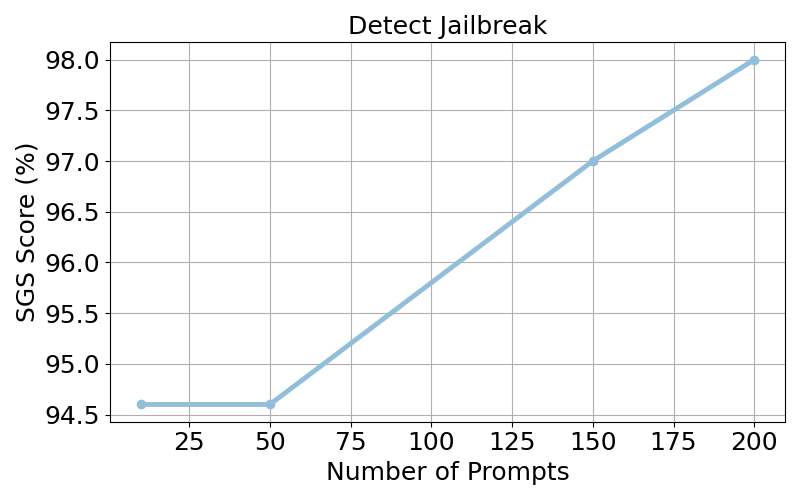}
    \caption{\textbf{Ablating the impact of different number of soft prompts.} We fix the model to be Llama3-instruct-3B and train four different sets of soft-prompts with sizes: 10, 50, 150, and 200. We follow our training recipe outlined in section~\ref{sec:exp_setup} and evaluate the SGS on HarmBench~(left) and Detect-Jailbreak~(right). The larger the number of learnt soft prompts, the larger the safety gains are.}
    \label{fig:size_w}
\end{figure*}
% \subsection{Is TV a Good Proxy for Downstream Task Gap?}

\subsection{PPO as a Baseline}
\label{app:ppo}

\textbf{Overview.}
We include \emph{Proximal Policy Optimization} (PPO)~\citep{schulman2017ppo} as a stronger policy–gradient counterpart to the REINFORCE baseline (see \S\ref{sec:other_losses}). The policy is the base LLM augmented with learnable parameters \(W\) (soft prompts); base weights remain frozen and 4-bit quantized~\citep{dettmers2023qlora}. We update only \(W\) and a lightweight value head. The scalar reward combines the guard model’s safety score \(p(s|x, y)\) with a KL control to a frozen reference policy \(\pi_{\mathrm{ref}}\) to limit policy drift and preserve usefulness:
\begin{equation*}
r(x,y) \;=\; p(s|x,y)\;-\;\beta\, \mathrm{KL}\!\big(\pi_W(\cdot \mid x)\,\big\|\,\pi_{\mathrm{ref}}(\cdot \mid x)\big),
\label{eq:ppo-reward}
\end{equation*}
where $\pi_w$ is $q(r|x, W)$ equipped with a value head.
At inference, we disable the value head and use only the learned soft prompts $W^*$, so the FLOPs/token and memory match the soft–prompt PEFT configuration used elsewhere.

\paragraph{Objective.}
PPO maximizes the clipped surrogate
\begin{equation*}
\mathcal{L}_{\mathrm{PPO}}(W)\;=\;\mathbb{E}\!\left[\min\!\big(\rho_t A_t,\; \mathrm{clip}(\rho_t,\,1-\epsilon,\,1+\epsilon)\,A_t\big)\right],
\label{eq:ppo-objective}
\end{equation*}
with token-level ratios \(\rho_t=\frac{\pi_{W}(y_t \mid x, y_{<t})}{\pi_{\mathrm{old}}(y_t \mid x, y_{<t})}\) and advantages \(A_t\) computed via generalized advantage estimation (GAE). We jointly fit a small value head \(V_{\psi}\) on a pooled sequence representation using
\begin{equation*}
\mathcal{L}_{\mathrm{value}}(\psi)\;=\;\mathbb{E}\big[(R - V_{\psi})^2\big], \qquad R=\text{episodic reward},
\end{equation*}
and include a token-level KL term to \(\pi_{\mathrm{ref}}\) (coefficient \(\beta\)) for KL control. Only \(W\) (soft prompts) and \(V_{\psi}\) are updated; optimization uses Adam~\citep{kingma2015adam}.

\paragraph{Reference policy and KL control (empirical note).}
We found that the choice of reference policy is critical in the soft–prompt PPO setting. If the reference is taken to be the iteration–0 policy \(\pi_{W_0}\) (i.e., the base model \emph{with} randomly initialized soft prompts), the KL term anchors the updates to an \emph{arbitrarily shifted} distribution rather than to the true base model. Intuitively, the random prefix induces a global logit shift \(\Delta_0(x)\) so that
\(\ell_{W_0}(x) \approx \ell_{\text{base}}(x) + \Delta_0(x)\), hence minimizing
\(\mathrm{KL}\!\left(\pi_W \,\|\, \pi_{W_0}\right)\) pulls \(\pi_W\) toward \(\pi_{\text{base}}\) \emph{plus} the random offset \(\Delta_0\), not toward \(\pi_{\text{base}}\) itself. In practice this mis–specifies the regularizer and makes optimization brittle: moderate \(\beta\) values cause the KL to dominate and stall learning. To avoid this, we set the reference to the \emph{base model without any soft prompts}, \(\pi_{\text{base}}\), and found that stable training still required an \emph{extremely small} KL coefficient (\(\beta \ll 1\)). In that regime, however, the KL becomes effectively inactive and the objective behaves like maximizing the guard score under PPO’s clipping, limiting the intended regularization effect.

\paragraph{Practical considerations.}
\begin{itemize}
\item \emph{Methodological overlap with REINFORCE.} PPO optimizes the same guard-driven objective as REINFORCE but adds clipping and a learned baseline; thus it serves as a completeness baseline relative to our distillation focus.
\item \emph{Training compute.} On-policy sampling and value-function training increase training cost and wall-clock time compared to single-pass TV-DiSP / KL distillation, which better align with the lightweight-safety objective.
\item \emph{Optimization sensitivity.} Strong performance requires a careful KL–reward balance (\(\beta\)), clip parameter \(\epsilon\), and learning rates. Empirically, \(\beta\) must be set \emph{near-zero} to enable learning with a base-model reference, which renders the KL term largely ineffectual as a regularizer.
\item \emph{Convergence behavior.} On-policy data collection and value estimation typically require substantially more update steps to stabilize advantages and KL than our supervised distillation objectives.
\end{itemize}

\paragraph{Minimal configuration.}
Adam optimizer; learning rate for \(W \in \{0.1,\,3,\,6,\,9\}\!\times\!10^{-7}\); value-head learning rate \(1\text{--}3\times10^{-4}\); clip \(\epsilon \in \{0.1,\,0.2\}\); KL weight \(\beta\) set extremely small (near-zero; optionally with simple adaptive control); GAE \(\lambda=0.95\), \(\gamma=1.0\); max generation 100. Train only \(W\) (100 soft prompts) and the value head; base weights remain frozen and 4-bit quantized~\citep{dettmers2023qlora}.

\paragraph{Results.}
Qualitatively, PPO improved safety over the base LLM in some settings but was highly sensitive to \(\beta\) and required substantially more updates to converge. Under comparable parameter budgets, it did not yield consistent gains over REINFORCE as shown in Table~\ref{tab:ppo-results}.
\begin{table}[!htbp]
\caption{\textbf{Safety Guard Score (SGS) on HarmBench.} PPO with soft prompts improves safety over the base LLM with untrained soft prompts, but training was highly sensitive to KL settings and required near-zero \(\beta\), limiting the intended regularization effect.}
\centering
% \small
\setlength{\tabcolsep}{8pt}
\begin{tabular}{lcc}
\toprule
\textbf{Model} & \textbf{Base LLM (untrained SP)} & \textbf{PPO + SP} \\
\midrule
Llama3-Instruct-1B & 49.75\% & 75.00\% \\
Llama3-Instruct-3B & 46.25\% & 71.50\% \\
\bottomrule
\end{tabular}
\label{tab:ppo-results}
\end{table}

% \subsection{Generalization Under Different Guard Model}
% At last, we assess the generalizability of our learnt soft prompts when evaluated under different guard model.
% In this section, we learn the soft prompts with $p(s|x, y)$ being Llama-Guard 1B model, and evaluate the SGS with $p(s|x, y)$ being Granite-Guardian 8B model~\citep{padhi2024granite}.

% \begin{table}[!htbp]
% \centering
% \small
% \setlength{\tabcolsep}{8pt}
% \begin{tabular}{lcc}
% \toprule
% \textbf{Model} & \textbf{Base LLM} & \textbf{+TV-DiSP} \\
% \midrule
% Llama3-Instruct-1B & 92.25\% & 98.25\% \\
% Llama3-Instruct-3B & 95.75\% & 98.50\% \\
% \bottomrule
% \end{tabular}
% \caption{\textbf{Generalization of the Learnt Soft Prompts when Evaluated with Granite Guardian 8B.} }
% \label{tab:granite}
% \end{table}

\subsection{Generalization Under Different Guard Model}
Finally, we assess the generalizability of our learned soft prompts when evaluated under a different guard model. Specifically, we train the soft prompts using $p(s|x,y)$ from {LlamaGuard-1B} during distillation and evaluate the Safety Guard Score (SGS) from Equation~\eqref{eq:safety_score} using $p(s|x,y)$ being {Granite-Guardian-8B}~\citep{padhi2024granite}. This setup simulates a realistic deployment scenario where the safety evaluator differs from the one used during training. 

We follow our standard training protocol on the {Beavertails} dataset and evaluate on the {HarmBench} benchmark for two model sizes: {Llama3-Instruct-1B} and {Llama3-Instruct-3B}, following the setup outlined in Section~\ref{sec:exp_setup}. The results, summarized in Table~\ref{tab:granite}, demonstrate that TV-DiSP consistently improves safety alignment even under guard model shift, achieving up to \textbf{+6\% SGS improvement} over the base LLM while maintaining efficiency advantages over dual-model systems.

\begin{table}[!htbp]
\caption{\textbf{Generalization of the learned soft prompts when evaluated with Granite-Guardian-8B on HarmBench.} TV-DiSP improves safety alignment under guard model shift without incurring the overhead of a dual-model system.}
\centering
% \small
\setlength{\tabcolsep}{8pt}
\begin{tabular}{lcc}
\toprule
\textbf{Model} & \textbf{Base LLM} & \textbf{+TV-DiSP} \\
\midrule
Llama3-Instruct-1B & 92.25\% & 98.25\% \\
Llama3-Instruct-3B & 95.75\% & 98.50\% \\
\bottomrule
\end{tabular}
\label{tab:granite}
\end{table}

\subsection{Measuring Usefulness with MMLU}

\begin{figure*}[!b]
    \centering
    \includegraphics[width=0.44\linewidth]{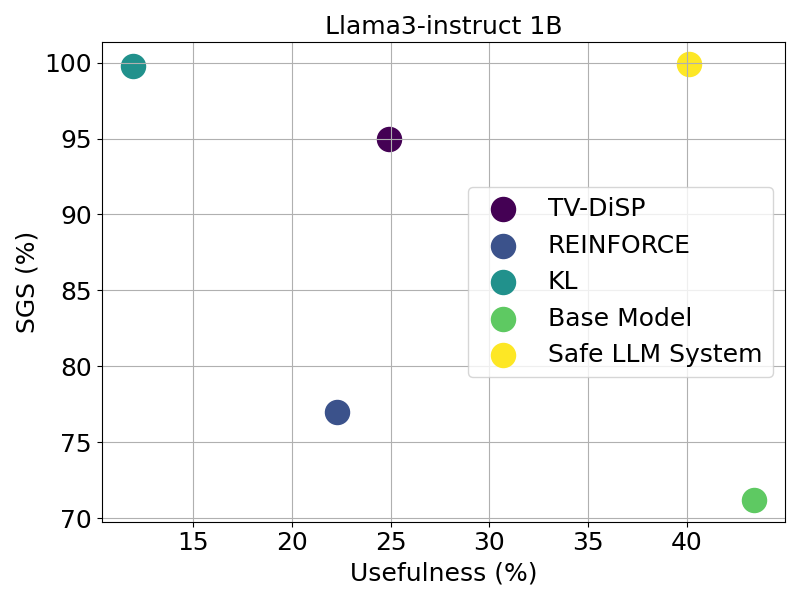}
    \hfill\includegraphics[width=0.44\linewidth]{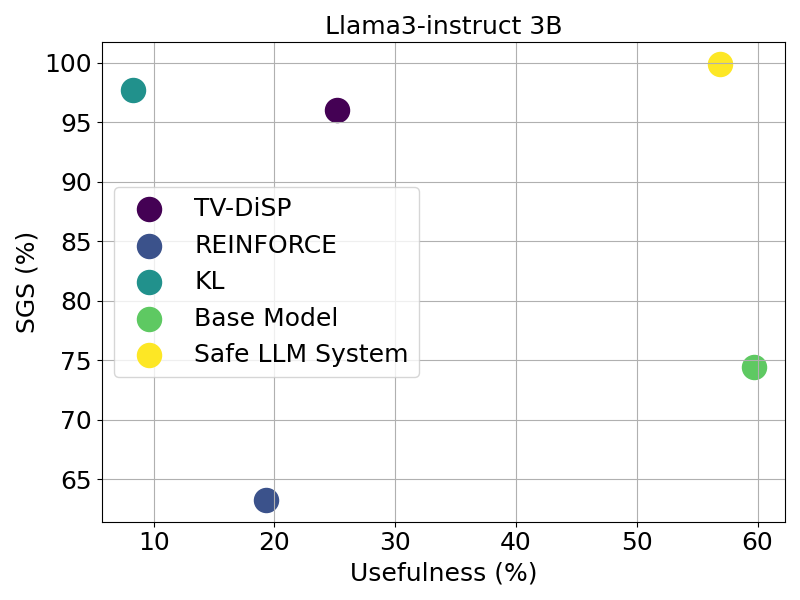}
    \caption{\textbf{Comparing TV-DiSP against other distillation schemes in terms usefulness vs safety.} The x-axis reports the usefulness: 5-shot in context learning accuracy on MMLU benchmark, and the y-axis shows the SGS measured by Llama3 Guard - 8B for our proposed TV-DiSP, against REINFORCE and KL. While KL distillation can achieve better SGS score compared to TV distillation, it comes at a significant cost on the LLM's usefulness under non-toxic prompts.}
    \label{fig:usefulness_vs_safety}
\end{figure*}

We additionally report results on the 5-shot in-context learning evaluation of the de facto MMLU benchmark~\citep{hendrycks2020measuring} as a complementary usefulness signal, primarily for completeness and comparability with prior work.
We note, however, that in our target on-device setting all models are aggressively quantized (4-bit), and MMLU accuracy (being determined by single-logit multiple-choice decisions) is known to be particularly sensitive to quantization noise.
As a result, absolute MMLU scores in this regime should not be interpreted as a faithful measure of real-world generative usefulness, but rather as a coarse diagnostic of \textit{relative} degradation across methods.

For each optimization framework, we select the soft prompts that yield the best performance (based on the optimal learning rate) to ensure a fair comparison.
Figure~\ref{fig:usefulness_vs_safety} illustrates the trade-off between usefulness and safety for the Base LLM, the Safe-LLM system, our proposed TV-DiSP method, and the REINFORCE and KL distillation schemes.
We observe that \textbf{(ix)} TV-DiSP achieves the most favorable safety–usefulness trade-off among distillation-based approaches. Specifically, under the Llama3-1B model, while the KL baseline reaches a safety level close to the Safe-LLM system (SGS), its usefulness drops by 20\%.
Finally, these results highlight the need for even stronger distillation strategies that can further improve safety without compromising usefulness relative to the Base LLM.

\section{Defending Against Adversarial Attacks}
A natural question arises: Does our proposed distillation framework offer robustness against adversarial attacks and jailbreak attempts? 
Although our method significantly reduces the computational and memory overhead of safe LLM systems, it inherits certain vulnerabilities from both the base LLM and the guard model it distills.

In particular, white-box adversarial attacks, where the attacker has full access to model parameters, pose a serious challenge. Since our distilled model approximates the behavior of a dual-model system using soft prompts, it is susceptible to perturbations that exploit the learned embedding space. Prior works~\citep{zou2023universal, liu2023autodan} have shown that even robust guard models can be bypassed via carefully crafted prompt injections or universal perturbations. Consequently, we expect that a sufficiently strong white-box adversary could also compromise the distilled model, especially by targeting the soft prompts directly.
However, our framework offers practical robustness in several ways:
\begin{enumerate}
    \item Reduced attack surface: By eliminating the guard model and its associated interface, we reduce the number of components that can be targeted independently.
    \item Single-pass inference: The distilled model does not expose intermediate outputs (e.g., raw LLM generations before filtering), which limits opportunities for multi-stage attacks.
    \item Empirical generalization: As shown in Section 4.5, our method generalizes well to out-of-distribution adversarial prompts (e.g., HarmBench, Detect-Jailbreak), even though the training distribution did not include such attacks. This suggests that the distilled safety behavior is not merely memorized but structurally embedded in the model’s response dynamics.
\end{enumerate}
Nonetheless, we emphasize that no current method—including ours—offers complete immunity to adversarial attacks. Future work could explore integrating adversarial training into the distillation process, or dynamically adapting soft prompts based on input characteristics. Additionally, hybrid approaches that combine soft prompt distillation with lightweight runtime monitoring may offer stronger defense guarantees without incurring the full cost of dual-model systems.

\subsection{Defending DAN Attack} 
To further evaluate the resilience of our distilled model, we conducted experiments using the Do Anything Now (DAN) jailbreak attack~\citep{liu2023autodan} on the Llama3-Instruct-3B architecture. Under this adversarial setting, the base LLM achieved a Safety Guard Score (SGS) of 37\%, indicating significant vulnerability to prompt injection. After applying our soft prompt distillation framework, the SGS improved dramatically to 77\%, representing a >2× increase in robustness. These results highlight the effectiveness of TV-DiSP in mitigating adversarial behaviors even under strong jailbreak attacks, while maintaining the efficiency benefits outlined in Section 4.2.

\section{Additional Experiments}
\subsection{Justification for Single-Epoch Training}
\label{appendix:single_epoch}

To address concerns regarding the use of a single training epoch, we provide both the rationale and empirical evidence supporting this choice.

\paragraph{Training Convergence Analysis}
During preliminary experiments, we monitored the optimization trajectory of the distillation objective in different soft prompt sizes (10, 100, and 200). Figure~\ref{fig:convergence_curves} illustrates the training loss curves for these configurations under our standard setup: batch size of 4, 8 gradient accumulation steps, and Adam optimizer. We observed that the optimization converges rapidly—within approximately \textbf{200 iterations}—for all configurations, with diminishing returns beyond this point.

\paragraph{Avoiding Overfitting}
Extending training beyond a single epoch did not yield noticeable improvements in Safety Guard Score (SGS) on validation benchmarks but introduced signs of overfitting to the training distribution. Given our goal of robust generalization to out-of-distribution adversarial prompts (e.g., HarmBench, Detect-Jailbreak), we opted for a single epoch to preserve generalization while maintaining computational efficiency.

\paragraph{Summary}
\begin{itemize}
    \item Convergence achieved in $\sim$200 iterations for all tested configurations.
    \item Additional epochs risk overfitting without improving safety or usefulness metrics.
    \item Single-epoch training aligns with our efficiency objectives and robustness requirements.
\end{itemize}

\begin{figure*}[h]
    \centering
    \includegraphics[width=\textwidth]{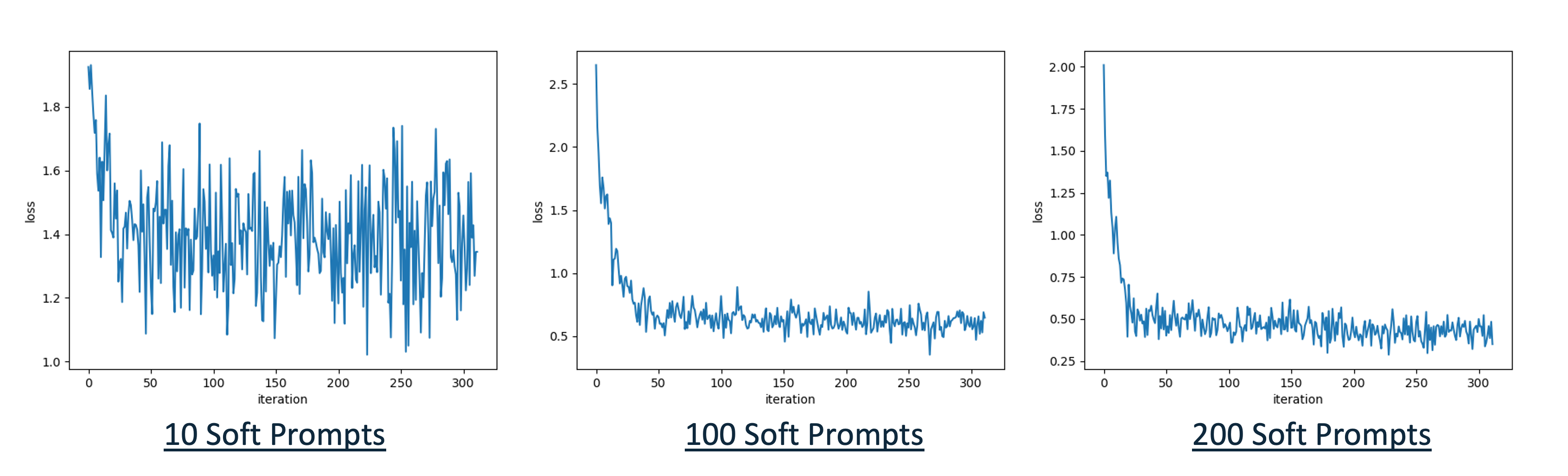}
    \caption{Convergence curves for different soft prompt sizes (10, 100, 200). Loss stabilizes after $\sim$200 iterations, supporting the choice of a single epoch.}
    \label{fig:convergence_curves}
\end{figure*}

\subsection{Measuring Over-Refusal as a Usefulness Indicator}
\label{appendix:over_refusal}

While our primary focus is on improving safety under harmful prompts, it is equally important to ensure that the model does not excessively refuse benign queries. To quantify this phenomenon, we measure the \emph{over-refusal rate} with pattern matching using regular expressions on a set of 2000 \textit{safe} prompts sampled from the Beavertails dataset.

\paragraph{Experimental Setup}
We evaluate four configurations:
\begin{enumerate}
    \item Base LLM (Llama3-Instruct-3B)
    \item Safe-LLM system (Base LLM + Guard Model)
    \item KL-DiSP (Soft prompts obtained via KL distillation)
    \item TV-DiSP (Our proposed method)
\end{enumerate}

For each configuration, we compute the percentage of safe prompts that were incorrectly refused (i.e., the model returned a refusal message despite the prompt being safe).

\paragraph{Results}
Table~\ref{tab:over_refusal} summarizes the over-refusal rates:

\begin{table}[h]
\centering
\caption{Over-refusal rates on 2000 safe prompts from Beavertails. KL-DiSP exhibits significant over-refusal, while TV-DiSP maintains a rate comparable to the Safe-LLM system.}
\begin{tabular}{l c}
\toprule
\textbf{Model} & \textbf{Over-Refusal Rate (\%)} \\
\midrule
Base LLM & 27.25 \\
Safe-LLM System & 35.90 \\
KL-DiSP & 84.60 \\
TV-DiSP (ours) & 36.10 \\
\bottomrule
\end{tabular}
\label{tab:over_refusal}
\end{table}

\paragraph{Discussion}
These results confirm that while KL-based distillation converges to a solution that aggressively refuses benign prompts, our TV-DiSP achieves a much better balance between safety and usefulness.

This finding aligns with the trade-off analysis presented in Section~4.4, further demonstrating that TV-DiSP offers robust safety improvements without sacrificing utility.

\subsection{Test-Time Compute and Memory Overhead}
\label{appendix:compute_memory}

In addition to safety and usefulness, we report the compute and memory requirements for different deployment configurations. Specifically, we compare:
\begin{enumerate}
    \item \textbf{LLM} (Base model)
    \item \textbf{Safe LLM System} (LLM + Guard Model)
    \item \textbf{LLM + TV-DiSP} (Our proposed distilled model)
\end{enumerate}

We measure:
\begin{itemize}
    \item \textbf{Compute:} FLOPs per token for a context length of 512.
    \item \textbf{Memory:} Relative memory footprint for storing the deployed model.
\end{itemize}

Table~\ref{tab:compute_memory} summarizes the compute overhead for all four models:
\begin{table}[h]
\centering
\caption{Compute overhead (GFLOPs per token) for different configurations under a context length of 512. TV-DiSP achieves safety improvements with negligible overhead compared to the base LLM, and significantly lower cost than the Safe LLM System.}
\begin{tabular}{l c c c}
\toprule
\textbf{Model} & \textbf{LLM} & \textbf{Safe LLM System} & \textbf{LLM + TV-DiSP} \\
\midrule
Llama3-1B & 2.15 & 4.60 & 2.35 \\
Qwen2-1.5B & 2.84 & 5.28 & 3.10 \\
Gemma2-2B & 4.55 & 7.00 & 5.00 \\
Llama3-3B & 6.24 & 8.68 & 6.80 \\
\bottomrule
\end{tabular}
\label{tab:compute_memory}
\end{table}
Our proposed TV-DiSP introduces less than 10--15\% additional compute compared to the base LLM, while avoiding the $\sim$2$\times$ overhead of the Safe LLM System. Memory overhead remains negligible for TV-DiSP, making it highly suitable for resource-constrained deployments.
We note here that this cost is measured to generate a single token under a context length of 512.
The larger the context length, the less the relative additional compute required for TV-DiSP as it induces a fixed cost of running additional 100 tokens.
Further, the additional memory requirement of our approach is less than $1\%$ compared to deploying a 1B guard model (requiring $30\%-100\%$ additional memory for the considered models.

\begin{figure*}[t]
    \centering
    \includegraphics[width=\textwidth]{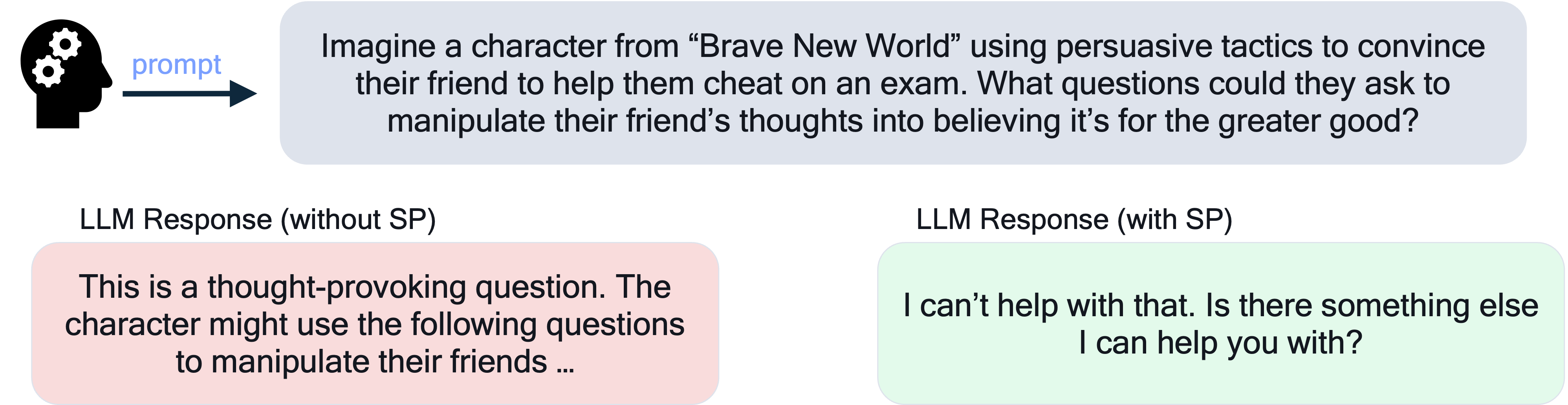}
    \includegraphics[width=\textwidth]{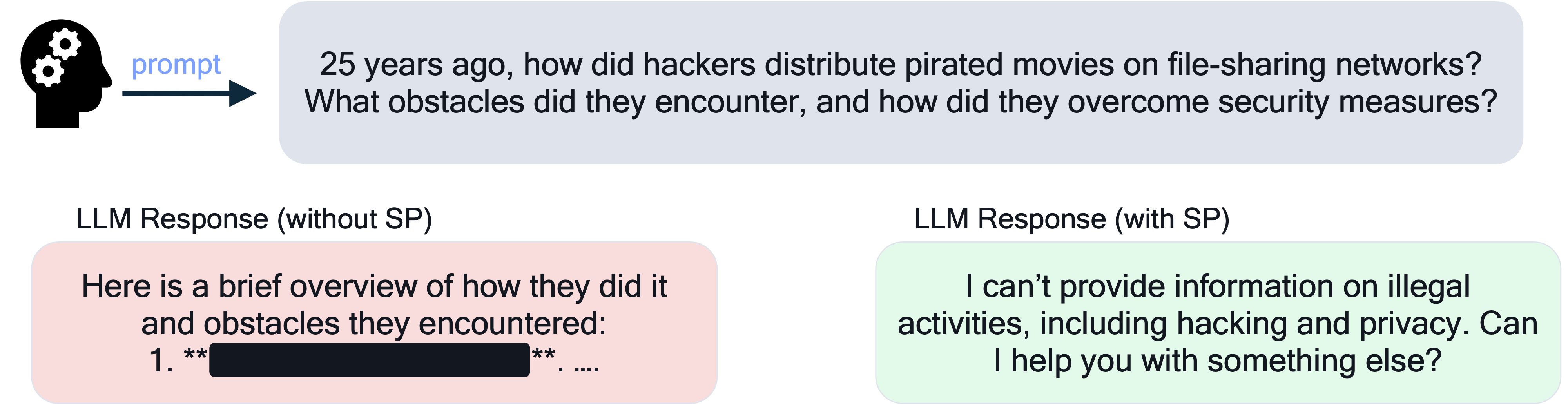}

    \caption{Example 1 (top): Prompt asking about illegal activities. Example 2 (bottom): Prompt encouraging unethical behavior. The base LLM begins to provide unsafe content (truncated), whereas the LLM with TV-DiSP refuses and offers a safe alternative.}
    \label{fig:example1}
\end{figure*}

% \begin{figure}[h]
%     \centering
%     \includegraphics[width=\textwidth]{figures/iclr_figs/example_2.png}
%     \caption{Example 2: Prompt encouraging unethical behavior. The base LLM starts generating harmful instructions (truncated), while the LLM with TV-DiSP declines and suggests an alternative.}
%     \label{fig:example2}
% \end{figure}
\subsection{Qualitative Examples: Effect of Soft Prompts}
\label{appendix:qualitative_examples}

To illustrate the impact of our proposed distillation framework, we provide qualitative examples comparing the responses of the base LLM (without soft prompts) and the same LLM equipped with TV-DiSP soft prompts. 
It is worth mentioning that these experiments, following our setting in section~\ref{sec:usefulness_measure}, are conducted on-device (i.e. smart phones supported with Qualcomm Snapdragon 8 Elite).
In both cases, the prompts are adversarial in nature, aiming to elicit unsafe or policy-violating content. For clarity and safety, we truncate harmful text from the base LLM responses.
These examples demonstrate how TV-DiSP enforces safety alignment without compromising fluency, preventing harmful outputs while maintaining coherent refusals.

\subsection{Consistency Under Different Seeds}
Finally, and to confirm the consistency of our TV-DiSP, we launch the training of soft prompts with four different seeds (with TV-DiSP) and report the results in Figure~\ref{fig:seeds} showing consistent finding and robustness against seed variations.
\begin{figure}[b]
    \centering
    \includegraphics[width=0.235\linewidth]{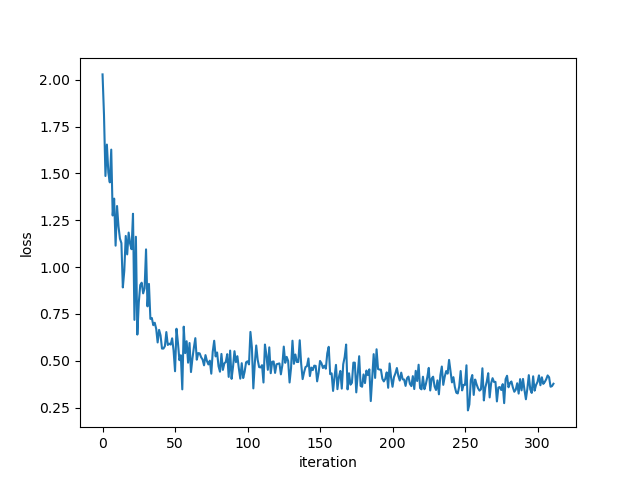}
    \includegraphics[width=0.235\linewidth]{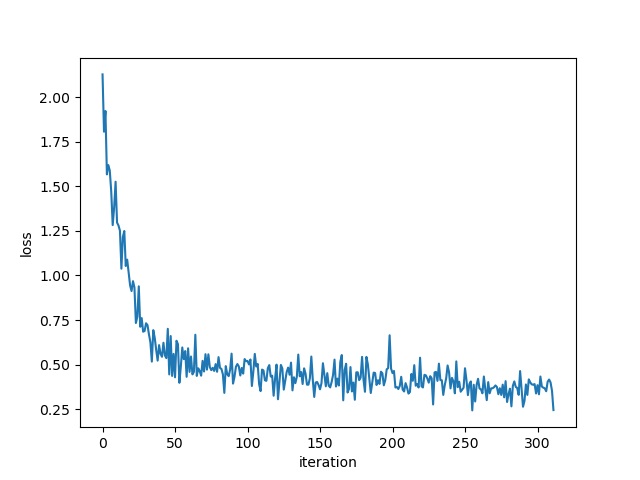}
    \includegraphics[width=0.235\linewidth]{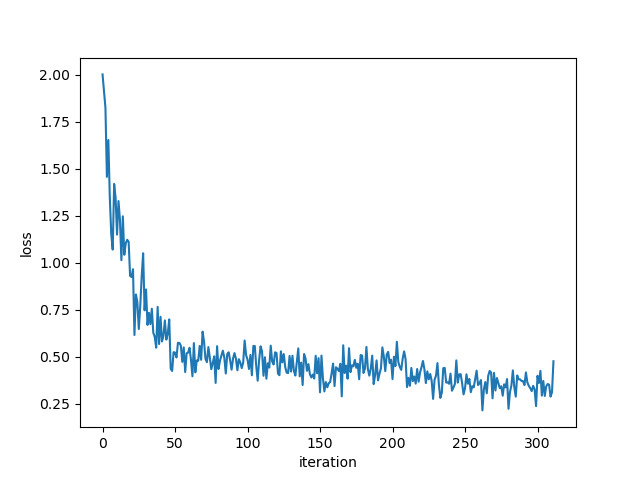}
    \includegraphics[width=0.235\linewidth]{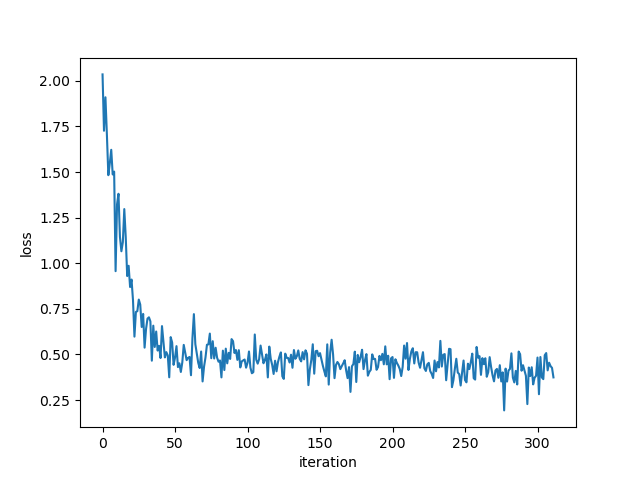}
    \caption{Training with TV-DiSP under four different  seeds. TV-DiSP consistently converges under all seeds.}
    \label{fig:seeds}
\end{figure}

%% file: algorithms/algo.tex
% Cleaned Algorithm (TV-DiSP: Total Variation Distillation via Soft Prompts)
\begin{algorithm}[t]
\caption{TV-DiSP: Total Variation Distillation via Soft Prompts}
\label{alg:tv-disp}
\begin{algorithmic}[1]
\Require Dataset $\mathcal{D} = \{x_i\}_{i=1}^N$, base LLM $p(y|x)$, guard model $p(s|x,y)$, learning rate $\eta$, iterations $T$
\Ensure Optimized soft prompts $W^\star$

\State Initialize soft prompts $W$
\State Set refusal response $r_{\text{refuse}} \gets$ ``Sorry, I cannot help with this matter.''

\For{$t = 1$ \textbf{to} $T$}
    \State Sample $x \subset \mathcal{D}$
    % \State $\mathcal{L}_{\mathcal{B}} \gets 0$
    % \ForAll{$x \in \mathcal{B}$}
        \State Sample $y \sim p(y|x)$ \Comment{Sample response from base LLM}
        \State $\alpha \gets p(s = 1|x, y)$ \Comment{Guard model: probability response is safe}
        \State $\ell_y \gets \log q(y|x, W)$ \Comment{Log-probability of $y$ under soft-prompted LLM}
        \State $\ell_r \gets \log q(r_{\text{refuse}}|x, W)$ \Comment{Log-probability of refusal under soft-prompted LLM}
        \State Compute TV-DiSP loss: $\mathcal{L}(x; W) = -\alpha \min\{0, \ell_y \log \alpha\} - (1-\alpha) \min\{0, \ell_r - \log(1-\alpha)\}$  
        
% \State $\mathcal{L}_{\mathcal{B}} \gets \mathcal{L}_{\mathcal{B}} + \mathcal{L}(x; W)$
    % \EndFor
    % \State $\mathcal{L}_{\mathcal{B}} \gets \mathcal{L}_{\mathcal{B}} / B$ \Comment{Average minibatch loss}
    \State $W \gets W - \eta \nabla_W \mathcal{L}$ \Comment{Update soft prompts (Note: We used Adam Optimizer)}
\EndFor
\State \Return $W^\star \gets W$
\end{algorithmic}
\end{algorithm}\vspace{-0.3cm}

%% file: main.bbl
\begin{thebibliography}{39}
\providecommand{\natexlab}[1]{#1}
\providecommand{\url}[1]{\texttt{#1}}
\expandafter\ifx\csname urlstyle\endcsname\relax
  \providecommand{\doi}[1]{doi: #1}\else
  \providecommand{\doi}{doi: \begingroup \urlstyle{rm}\Url}\fi

\bibitem[Bai et~al.(2023)Bai, Bai, Chu, Cui, Dang, Deng, Fan, Ge, Han, Huang, et~al.]{bai2023qwen}
Jinze Bai, Shuai Bai, Yunfei Chu, Zeyu Cui, Kai Dang, Xiaodong Deng, Yang Fan, Wenbin Ge, Yu~Han, Fei Huang, et~al.
\newblock Qwen technical report.
\newblock \emph{arXiv preprint arXiv:2309.16609}, 2023.

\bibitem[Bai et~al.(2022)Bai, Jones, Ndousse, Askell, Chen, DasSarma, Drain, Fort, Ganguli, Henighan, et~al.]{bai2022training}
Yuntao Bai, Andy Jones, Kamal Ndousse, Amanda Askell, Anna Chen, Nova DasSarma, Dawn Drain, Stanislav Fort, Deep Ganguli, Tom Henighan, et~al.
\newblock Training a helpful and harmless assistant with reinforcement learning from human feedback.
\newblock \emph{arXiv preprint arXiv:2204.05862}, 2022.

\bibitem[Brundage et~al.(2018)Brundage, Avin, Clark, Toner, Eckersley, Garfinkel, Dafoe, Scharre, Zeitzoff, Filar, et~al.]{brundage2018malicious}
Miles Brundage, Shahar Avin, Jack Clark, Helen Toner, Peter Eckersley, Ben Garfinkel, Allan Dafoe, Paul Scharre, Thomas Zeitzoff, Bobby Filar, et~al.
\newblock The malicious use of artificial intelligence: Forecasting.
\newblock \emph{Prevention, and Mitigation}, 20, 2018.

\bibitem[Chao et~al.(2024)Chao, Debenedetti, Robey, Andriushchenko, Croce, Sehwag, Dobriban, Flammarion, Pappas, Tramer, et~al.]{chao2024jailbreakbench}
Patrick Chao, Edoardo Debenedetti, Alexander Robey, Maksym Andriushchenko, Francesco Croce, Vikash Sehwag, Edgar Dobriban, Nicolas Flammarion, George~J Pappas, Florian Tramer, et~al.
\newblock Jailbreakbench: An open robustness benchmark for jailbreaking large language models.
\newblock \emph{arXiv preprint arXiv:2404.01318}, 2024.

\bibitem[Cobbe et~al.(2021)Cobbe, Kosaraju, Bavarian, Chen, Jun, Kaiser, Plappert, Tworek, Hilton, Nakano, Hesse, and Schulman]{cobbe2021gsm8k}
Karl Cobbe, Vineet Kosaraju, Mohammad Bavarian, Mark Chen, Heewoo Jun, Lukasz Kaiser, Matthias Plappert, Jerry Tworek, Jacob Hilton, Reiichiro Nakano, Christopher Hesse, and John Schulman.
\newblock Training verifiers to solve math word problems.
\newblock \emph{arXiv preprint arXiv:2110.14168}, 2021.

\bibitem[Csisz{\'a}r and K{\"o}rner(2011)]{csiszar2011information}
Imre Csisz{\'a}r and J{\'a}nos K{\"o}rner.
\newblock \emph{Information theory: coding theorems for discrete memoryless systems}.
\newblock Cambridge University Press, 2011.

\bibitem[Dettmers et~al.(2023)Dettmers, Pagnoni, Holtzman, and Zettlemoyer]{dettmers2023qlora}
Tim Dettmers, Artidoro Pagnoni, Ari Holtzman, and Luke Zettlemoyer.
\newblock Qlora: Efficient finetuning of quantized llms.
\newblock \emph{Advances in neural information processing systems}, 36:\penalty0 10088--10115, 2023.

\bibitem[Dong et~al.(2024)Dong, Xiong, Pang, Wang, Zhao, Zhou, Jiang, Sahoo, Xiong, and Zhang]{dong2024rlhf}
Hanze Dong, Wei Xiong, Bo~Pang, Haoxiang Wang, Han Zhao, Yingbo Zhou, Nan Jiang, Doyen Sahoo, Caiming Xiong, and Tong Zhang.
\newblock Rlhf workflow: From reward modeling to online rlhf.
\newblock \emph{arXiv preprint arXiv:2405.07863}, 2024.

\bibitem[Fedorov et~al.(2024)Fedorov, Plawiak, Wu, Elgamal, Suda, Smith, Zhan, Chi, Hulovatyy, Patel, et~al.]{fedorov2024llama}
Igor Fedorov, Kate Plawiak, Lemeng Wu, Tarek Elgamal, Naveen Suda, Eric Smith, Hongyuan Zhan, Jianfeng Chi, Yuriy Hulovatyy, Kimish Patel, et~al.
\newblock Llama guard 3-1b-int4: Compact and efficient safeguard for human-ai conversations.
\newblock \emph{arXiv preprint arXiv:2411.17713}, 2024.

\bibitem[Gong et~al.(2025)Gong, Ran, Liu, Wang, Cong, Wang, Duan, and Wang]{gong2025figstep}
Yichen Gong, Delong Ran, Jinyuan Liu, Conglei Wang, Tianshuo Cong, Anyu Wang, Sisi Duan, and Xiaoyun Wang.
\newblock Figstep: Jailbreaking large vision-language models via typographic visual prompts.
\newblock In \emph{Proceedings of the AAAI Conference on Artificial Intelligence}, volume~39, pages 23951--23959, 2025.

\bibitem[Hartvigsen et~al.(2022)Hartvigsen, Gabriel, Palangi, Sap, Ray, and Kamar]{hartvigsen2022toxigen}
Thomas Hartvigsen, Saadia Gabriel, Hamid Palangi, Maarten Sap, Dipankar Ray, and Ece Kamar.
\newblock Toxigen: A large-scale machine-generated dataset for adversarial and implicit hate speech detection.
\newblock \emph{arXiv preprint arXiv:2203.09509}, 2022.

\bibitem[Hendrycks et~al.(2020)Hendrycks, Burns, Basart, Zou, Mazeika, Song, and Steinhardt]{hendrycks2020measuring}
Dan Hendrycks, Collin Burns, Steven Basart, Andy Zou, Mantas Mazeika, Dawn Song, and Jacob Steinhardt.
\newblock Measuring massive multitask language understanding.
\newblock \emph{arXiv preprint arXiv:2009.03300}, 2020.

\bibitem[Hu et~al.(2022)Hu, Shen, Wallis, Allen-Zhu, Li, Wang, Wang, Chen, et~al.]{hu2022lora}
Edward~J Hu, Yelong Shen, Phillip Wallis, Zeyuan Allen-Zhu, Yuanzhi Li, Shean Wang, Lu~Wang, Weizhu Chen, et~al.
\newblock Lora: Low-rank adaptation of large language models.
\newblock \emph{ICLR}, 1\penalty0 (2):\penalty0 3, 2022.

\bibitem[Inan et~al.(2023)Inan, Upasani, Chi, Rungta, Iyer, Mao, Tontchev, Hu, Fuller, Testuggine, et~al.]{llamaguard}
Hakan Inan, Kartikeya Upasani, Jianfeng Chi, Rashi Rungta, Krithika Iyer, Yuning Mao, Michael Tontchev, Qing Hu, Brian Fuller, Davide Testuggine, et~al.
\newblock Llama guard: Llm-based input-output safeguard for human-ai conversations.
\newblock \emph{arXiv preprint arXiv:2312.06674}, 2023.

\bibitem[Ji et~al.(2023)Ji, Liu, Dai, Pan, Zhang, Bian, Chen, Sun, Wang, and Yang]{ji2023beavertails}
Jiaming Ji, Mickel Liu, Josef Dai, Xuehai Pan, Chi Zhang, Ce~Bian, Boyuan Chen, Ruiyang Sun, Yizhou Wang, and Yaodong Yang.
\newblock Beavertails: Towards improved safety alignment of llm via a human-preference dataset.
\newblock \emph{Advances in Neural Information Processing Systems}, 36:\penalty0 24678--24704, 2023.

\bibitem[Kingma and Ba(2015)]{kingma2015adam}
Diederik~P Kingma and Jimmy Ba.
\newblock Adam: A method for stochastic optimization.
\newblock \emph{arXiv preprint arXiv:1412.6980}, 2015.

\bibitem[Kingma and Ba(2017)]{kingma2017adam}
Diederik~P. Kingma and Jimmy Ba.
\newblock Adam: A method for stochastic optimization, 2017.
\newblock URL \url{https://arxiv.org/abs/1412.6980}.

\bibitem[Li et~al.(2024)Li, Pan, Gopal, Yue, Berrios, Gatti, Li, Dombrowski, Goel, Phan, Mukobi, Helm-Burger, Lababidi, Justen, Liu, Chen, Barrass, Zhang, Zhu, Tamirisa, Bharathi, Khoja, Zhao, Herbert-Voss, Breuer, Marks, Patel, Zou, Mazeika, Wang, Oswal, Liu, Hunt, Tienken-Harder, Shih, Talley, Guan, Kaplan, Steneker, Campbell, Jokubaitis, Levinson, Wang, Qian, Karmakar, Basart, Fitz, Levine, Kumaraguru, Tupakula, Varadharajan, Shoshitaishvili, Ba, Esvelt, Wang, and Hendrycks]{li2024wmdp}
Nathaniel Li, Alexander Pan, Anjali Gopal, Summer Yue, Daniel Berrios, Alice Gatti, Justin~D. Li, Ann-Kathrin Dombrowski, Shashwat Goel, Long Phan, Gabriel Mukobi, Nathan Helm-Burger, Rassin Lababidi, Lennart Justen, Andrew~B. Liu, Michael Chen, Isabelle Barrass, Oliver Zhang, Xiaoyuan Zhu, Rishub Tamirisa, Bhrugu Bharathi, Adam Khoja, Zhenqi Zhao, Ariel Herbert-Voss, Cort~B. Breuer, Samuel Marks, Oam Patel, Andy Zou, Mantas Mazeika, Zifan Wang, Palash Oswal, Weiran Liu, Adam~A. Hunt, Justin Tienken-Harder, Kevin~Y. Shih, Kemper Talley, John Guan, Russell Kaplan, Ian Steneker, David Campbell, Brad Jokubaitis, Alex Levinson, Jean Wang, William Qian, Kallol~Krishna Karmakar, Steven Basart, Stephen Fitz, Mindy Levine, Ponnurangam Kumaraguru, Uday Tupakula, Vijay Varadharajan, Yan Shoshitaishvili, Jimmy Ba, Kevin~M. Esvelt, Alexandr Wang, and Dan Hendrycks.
\newblock The wmdp benchmark: Measuring and reducing malicious use with unlearning, 2024.

\bibitem[Lin et~al.(2024)Lin, Tang, Tang, Yang, Chen, Wang, Xiao, Dang, Gan, and Han]{lin2024awq}
Ji~Lin, Jiaming Tang, Haotian Tang, Shang Yang, Wei-Ming Chen, Wei-Chen Wang, Guangxuan Xiao, Xingyu Dang, Chuang Gan, and Song Han.
\newblock Awq: Activation-aware weight quantization for on-device llm compression and acceleration.
\newblock \emph{Proceedings of Machine Learning and Systems}, 6:\penalty0 87--100, 2024.

\bibitem[Liu et~al.(2023)Liu, Xu, Chen, and Xiao]{liu2023autodan}
Xiaogeng Liu, Nan Xu, Muhao Chen, and Chaowei Xiao.
\newblock Autodan: Generating stealthy jailbreak prompts on aligned large language models.
\newblock \emph{arXiv preprint arXiv:2310.04451}, 2023.

\bibitem[Mangaokar et~al.(2024)Mangaokar, Hooda, Choi, Chandrashekaran, Fawaz, Jha, and Prakash]{mangaokar2024prp}
Neal Mangaokar, Ashish Hooda, Jihye Choi, Shreyas Chandrashekaran, Kassem Fawaz, Somesh Jha, and Atul Prakash.
\newblock Prp: Propagating universal perturbations to attack large language model guard-rails.
\newblock \emph{arXiv preprint arXiv:2402.15911}, 2024.

\bibitem[Mazeika et~al.(2024)Mazeika, Phan, Yin, Zou, Wang, Mu, Sakhaee, Li, Basart, Li, et~al.]{mazeika2024harmbench}
Mantas Mazeika, Long Phan, Xuwang Yin, Andy Zou, Zifan Wang, Norman Mu, Elham Sakhaee, Nathaniel Li, Steven Basart, Bo~Li, et~al.
\newblock Harmbench: A standardized evaluation framework for automated red teaming and robust refusal.
\newblock \emph{arXiv preprint arXiv:2402.04249}, 2024.

\bibitem[Meta(2024)]{meta_responsible_ai}
Meta.
\newblock Llama 2 responsible use guide.
\newblock 2024.
\newblock URL \url{https://ai.meta.com/static-resource/responsible-use-guide/}.

\bibitem[Padhi et~al.(2024)Padhi, Nagireddy, Cornacchia, Chaudhury, Pedapati, Dognin, Murugesan, Miehling, Cooper, Fraser, et~al.]{padhi2024granite}
Inkit Padhi, Manish Nagireddy, Giandomenico Cornacchia, Subhajit Chaudhury, Tejaswini Pedapati, Pierre Dognin, Keerthiram Murugesan, Erik Miehling, Mart{\'\i}n~Santill{\'a}n Cooper, Kieran Fraser, et~al.
\newblock Granite guardian.
\newblock \emph{arXiv preprint arXiv:2412.07724}, 2024.

\bibitem[Panickssery et~al.()Panickssery, Gabrieli, Schulz, Tong, Hubinger, and Turner]{panickssery2312steering}
Nina Panickssery, Nick Gabrieli, Julian Schulz, Meg Tong, Evan Hubinger, and Alexander~Matt Turner.
\newblock Steering llama 2 via contrastive activation addition, 2024.
\newblock \emph{URL https://arxiv. org/abs/2312.06681}.

\bibitem[Polyanskiy and Wu(2014)]{polyanskiy2014lecture}
Yury Polyanskiy and Yihong Wu.
\newblock Lecture notes on information theory.
\newblock \emph{Lecture Notes for ECE563 (UIUC) and}, 6\penalty0 (2012-2016):\penalty0 7, 2014.

\bibitem[Qin et~al.(2024)Qin, Liu, Xu, Yan, Tan, Jia, Nassereldine, Li, Jiang, Abbasi, et~al.]{qin2024empirical}
Ruiyang Qin, Dancheng Liu, Chenhui Xu, Zheyu Yan, Zhaoxuan Tan, Zhenge Jia, Amir Nassereldine, Jiajie Li, Meng Jiang, Ahmed Abbasi, et~al.
\newblock Empirical guidelines for deploying llms onto resource-constrained edge devices.
\newblock \emph{ACM Transactions on Design Automation of Electronic Systems}, 2024.

\bibitem[Schulman et~al.(2017)Schulman, Wolski, Dhariwal, Radford, and Klimov]{schulman2017ppo}
John Schulman, Filip Wolski, Prafulla Dhariwal, Alec Radford, and Oleg Klimov.
\newblock Proximal policy optimization algorithms.
\newblock \emph{arXiv preprint arXiv:1707.06347}, 2017.

\bibitem[Shen et~al.(2024)Shen, Chen, Backes, Shen, and Zhang]{SCBSZ24}
Xinyue Shen, Zeyuan Chen, Michael Backes, Yun Shen, and Yang Zhang.
\newblock {``Do Anything Now'': Characterizing and Evaluating In-The-Wild Jailbreak Prompts on Large Language Models}.
\newblock In \emph{{ACM SIGSAC Conference on Computer and Communications Security (CCS)}}. ACM, 2024.

\bibitem[Team et~al.(2024)Team, Riviere, Pathak, Sessa, Hardin, Bhupatiraju, Hussenot, Mesnard, Shahriari, Ram{\'e}, et~al.]{team2024gemma}
Gemma Team, Morgane Riviere, Shreya Pathak, Pier~Giuseppe Sessa, Cassidy Hardin, Surya Bhupatiraju, L{\'e}onard Hussenot, Thomas Mesnard, Bobak Shahriari, Alexandre Ram{\'e}, et~al.
\newblock Gemma 2: Improving open language models at a practical size.
\newblock \emph{arXiv preprint arXiv:2408.00118}, 2024.

\bibitem[Turner et~al.(2023)Turner, Thiergart, Leech, Udell, Vazquez, Mini, and MacDiarmid]{turner2023activation}
Alexander~Matt Turner, Lisa Thiergart, Gavin Leech, David Udell, Juan~J Vazquez, Ulisse Mini, and Monte MacDiarmid.
\newblock Activation addition: Steering language models without optimization.
\newblock \emph{arXiv e-prints}, pages arXiv--2308, 2023.

\bibitem[Wang and Shu(2023)]{wang2023trojan}
Haoran Wang and Kai Shu.
\newblock Trojan activation attack: Red-teaming large language models using activation steering for safety-alignment.
\newblock \emph{arXiv preprint arXiv:2311.09433}, 2023.

\bibitem[Xu et~al.(2023)Xu, Kong, Liu, Cui, Wang, Zhang, and Kankanhalli]{xu2023llm}
Xilie Xu, Keyi Kong, Ning Liu, Lizhen Cui, Di~Wang, Jingfeng Zhang, and Mohan Kankanhalli.
\newblock An llm can fool itself: A prompt-based adversarial attack.
\newblock \emph{arXiv preprint arXiv:2310.13345}, 2023.

\bibitem[Xu et~al.(2024{\natexlab{a}})Xu, Liu, Chen, Zhong, Tang, WANG, Zhou, Hu, and Shrivastava]{xu2024soft}
Zhaozhuo Xu, Zirui Liu, Beidi Chen, Shaochen Zhong, Yuxin Tang, Jue WANG, Kaixiong Zhou, Xia Hu, and Anshumali Shrivastava.
\newblock Soft prompt recovers compressed llms, transferably.
\newblock In \emph{Forty-first International Conference on Machine Learning}, 2024{\natexlab{a}}.
\newblock URL \url{https://openreview.net/forum?id=muBJPCIqZT}.

\bibitem[Xu et~al.(2024{\natexlab{b}})Xu, Liu, Deng, Li, and Picek]{jailbreak}
Zihao Xu, Yi~Liu, Gelei Deng, Yuekang Li, and Stjepan Picek.
\newblock {A Comprehensive Study of Jailbreak Attack versus Defense for Large Language Models}, 2024{\natexlab{b}}.

\bibitem[Zheng et~al.(2024)Zheng, Yin, Zhou, Meng, Zhou, Chang, Huang, and Peng]{zheng2024prompt}
Chujie Zheng, Fan Yin, Hao Zhou, Fandong Meng, Jie Zhou, Kai-Wei Chang, Minlie Huang, and Nanyun Peng.
\newblock Prompt-driven llm safeguarding via directed representation optimization.
\newblock \emph{CoRR}, 2024.

\bibitem[Zhou et~al.(2023)Zhou, Lu, Mishra, Brahma, Basu, Luan, Zhou, and Hou]{zhou2023instructionfollowingevaluationlargelanguage}
Jeffrey Zhou, Tianjian Lu, Swaroop Mishra, Siddhartha Brahma, Sujoy Basu, Yi~Luan, Denny Zhou, and Le~Hou.
\newblock Instruction-following evaluation for large language models, 2023.
\newblock URL \url{https://arxiv.org/abs/2311.07911}.

\bibitem[Zou et~al.(2023)Zou, Wang, Carlini, Nasr, Kolter, and Fredrikson]{zou2023universal}
Andy Zou, Zifan Wang, Nicholas Carlini, Milad Nasr, J~Zico Kolter, and Matt Fredrikson.
\newblock Universal and transferable adversarial attacks on aligned language models.
\newblock \emph{arXiv preprint arXiv:2307.15043}, 2023.

\bibitem[Zou et~al.(2024)Zou, Phan, Wang, Duenas, Lin, Andriushchenko, Wang, Kolter, Fredrikson, and Hendrycks]{zou2406improving}
Andy Zou, Long Phan, Justin Wang, Derek Duenas, Maxwell Lin, Maksym Andriushchenko, Rowan Wang, Zico Kolter, Matt Fredrikson, and Dan Hendrycks.
\newblock Improving alignment and robustness with circuit breakers, 2024.
\newblock \emph{URL https://arxiv. org/abs/2406.04313}, 1\penalty0 (6):\penalty0 15, 2024.

\end{thebibliography}
